\documentclass[11pt]{article}

\usepackage{amsthm}
\usepackage{jmlr2e}
\usepackage{mlmath}
\usepackage[shortlabels]{enumitem}
\usepackage{tabularx}
\usepackage{multirow}
\usepackage{booktabs}
\usepackage{subfigure}
\usepackage{caption}

\usepackage{xcolor}

\usepackage{makecell}

\newtheorem{thm}{Theorem}
\newtheorem{cor}{Corollary}

\newtheorem*{prop*}{Proposition}
\newtheorem{prop}{Proposition}
\newtheorem{rmk}{Remark}

\newtheorem*{theorem*}{Theorem}
\newtheorem*{lemma*}{Lemma}
\newtheorem*{property*}{Property}
\newtheorem*{remark*}{Remark}
\newtheorem*{assumption*}{Assumption}

\newcommand{\rnarr}{\textnormal{narr}}
\newcommand{\rwide}{\textnormal{wide}}

\newcommand{\rrank}{\textnormal{rank}}
\newcommand{\rdim}{\textnormal{dim}}
\newcommand{\rspan}{\textnormal{span}}

\newcommand{\TT}{\text{T}} 
 
\newcommand{\vep}{\varepsilon}

\title{Local Linear Recovery Guarantee of Deep Neural Networks at Overparameterization}
\begin{document}

\author{%
    \name Yaoyu Zhang${}^{a}$\thanks{Corresponding author} \email zhyy.sjtu@sjtu.edu.cn \\
    \name Leyang Zhang${}^{b}$ \email leyangz\_hawk@outlook.com \\
    \name Zhongwang Zhang${}^{a}$ \email 0123zzw666@sjtu.edu.cn \\
    \name Zhiwei Bai${}^{a}$ \email bai299@sjtu.edu.cn \\
    \addr ${}^a$School of Mathematical Sciences, Institute of Natural Sciences and MOE-LSC,
    Shanghai Jiao Tong University, Shanghai, 200240, China \\
    \addr ${}^b$ Department of Mathematics,
    University of Illinois Urbana-Champaign, Urbana, IL 61801,
    United States}

\editor{}
\date{\today}

\maketitle

\begin{abstract}
Determining whether deep neural network (DNN) models can reliably recover target functions at overparameterization is a critical yet complex issue in the theory of deep learning. To advance understanding in this area, we introduce a concept we term ``local linear recovery'' (LLR), a weaker form of target function recovery that renders the problem more amenable to theoretical analysis. In the sense of LLR, we prove that functions expressible by narrower DNNs are guaranteed to be recoverable from fewer samples than model parameters. Specifically, we establish upper limits on the optimistic sample sizes, defined as the smallest sample size necessary to guarantee LLR, for functions in the space of a given DNN. Furthermore, we prove that these upper bounds are achieved in the case of two-layer tanh neural networks. Our research lays a solid groundwork for future investigations into the recovery capabilities of DNNs in overparameterized scenarios.
\end{abstract}

\begin{keywords}
deep learning theory, recovery at overparameterization,
local linear recovery guarantee, optimistic sample size.
\end{keywords}

\section{Introduction}
Determining the requisite number of data points for a model to recover a target function is a fundamental problem in regression. This issue is particularly pertinent for deep neural network (DNN) models, where the question arises: Can these models recover target functions using fewer data points than their parameter count, thereby operating effectively in an overparameterized regime? Traditional wisdom suggests that a linear model with $M$ parameters typically needs $M$ data points to reconstruct a linear function. Particularly, in band-limited signal recovery, the Nyquist-Shannon sampling theorem posits that a periodic signal with a maximum frequency of $f$ (represented by $M=2f$ coefficients) can be perfectly reconstructed from $n \geqslant 2f$ uniformly sampled points \citep{shannon1984communication}. These precedents imply that linear models lack recovery guarantee when overparameterized. In contrast, nonlinear DNN models have been empirically shown to generalize effectively even when overparameterized \citep{zhang2016understanding}. Subsequent experiments reinforce this observation, indicating that DNNs can achieve zero generalization error in recovering target functions under overparameterization \citep{zhang2023optimistic}. This study delves into the theoretical underpinnings of target function recovery for DNNs, with the goal of establishing a recovery guarantee in overparameterized scenarios.

Within the spectrum of recovery guarantees, the global recovery guarantee—ensuring target recovery through comprehensive training—is the most robust and practically relevant, especially in overparameterized contexts. However, the intricate and highly nonlinear training dynamics of DNNs render this global guarantee theoretically elusive. To navigate this complexity, we propose a more attainable goal: establishing a weaker form of recovery guarantee for DNNs in overparameterized conditions. Specifically, we introduce the concept of local linear recovery (LLR) guarantee, which scales back the expectation from global recovery from a random initialization to local recovery in the vicinity of an optimal point within the parameter space. While local linear fitting of data is not practically feasible in general, the LLR-guarantee sheds light on the optimistic (i.e., best-possible) performance for global recovery of DNNs, suggesting that without an LLR-guarantee at overparameterization, stronger forms of recovery are unlikely.

Our main contributions are threefold: (i) We introduce the LLR-guarantee and formulate its theoretical framework for both differentiable and analytic models, demonstrating that linear models do not possess an LLR-guarantee when overparameterized. (ii) Employing the Embedding Principle \citep{zhang2021embedding,zhang2022embedding}, we derive upper bounds for the optimistic sample sizes—defined as the smallest sample sizes that ensure an LLR guarantee—for general DNNs. These bounds affirm the LLR-guarantee at overparameterization for all functions expressible by narrower DNNs. (iii) We pinpoint the exact optimistic sample sizes for two-layer fully-connected and convolutional tanh neural networks, which meet their respective upper bounds, thereby illustrating the exactitude of our upper bounds.

\section{Related works}

In our research, we introduce the concept of optimistic sample size, which quantifies the minimum number of training samples necessary to recover a target function under a best-possible condition. This approach diverges significantly from traditional sample complexity, which typically estimates the sample requirement for achieving a specified performance level in the worst-case scenarios \citep{shalev2014understanding}. Recognizing that DNNs often perform substantially better in practice than theoretical worst-case predictions suggest \citep{zhang2016understanding,zhang2021understanding}, our LLR analysis pioneers a framework for estimating sample sizes under the best-possible conditions. Furthermore, our empirical findings reveal that the practical performance of a finely-tuned DNN can approach this optimistic threshold, even in scenarios of substantial overparameterization. Moreover, as demonstrated in \cite{zhang2024implicit}, techniques such as dropout can further enhance the network's performance to recover target functions.

Our theoretical LLR framework, which is predicated on the linearization of DNNs, stands in stark contrast to linear analyses based on the NTK/lazy training/linear regime \citep{jacot2018neural,arora19fine,chizat2019lazy,luo2021phase}. The main differences are as follows: (i) NTK analysis linearizes around a random initial point, whereas LLR analysis considers linearization around a target point, i.e., a global minimizer with zero generalization error. (ii) NTK pertains to the linear training behavior of DNNs with large initial weights, whereas the optimistic sample sizes from LLR analysis are empirically linked to the performance of nonlinear training dynamics exhibited at small initializations. According to the phase diagrams in Refs. \citep{luo2021phase,zhou2022empirical}, NTK analysis corresponds to the linear regimes, while our LLR results associate to the condensed regimes, where nonlinear condensation dynamics are instrumental in achieving near-optimism performance.

The discovery of the embedding principle \citep{zhang2021embedding,zhang2022embedding,fukumizu2019semi,csimcsek2021geometry,bai2022embedding} has shed light on the analysis of critical points within the loss landscapes of DNNs and has forged connections between the loss landscapes of networks of varying widths. Notably, \cite{zhang2022embedding} introduces a comprehensive suite of critical embedding operators that map the parameter space of a narrower DNN to that of a wider one while preserving both the output function and the criticality of the network. This powerful analytical tool further illuminates the hierarchical model rank structure inherent in DNNs. Leveraging these critical embeddings, which inherently maintain the model rank, we derive in this work an upper bound for the optimistic sample size applicable to general DNNs.

\section{Theory of local linear recovery}
\subsection{Assumptions and definitions}
\begin{assumption*}
	\textcolor{black}{\\
		(i) We consider differentiable (w.r.t. both inputs and parameters) models with $1$-d output $f_{\vtheta}(\cdot)=f(\cdot;\vtheta):\sR^d\times\sR^M\to \sR$. \\
		(ii) We consider target functions $f^*$ expressible by the model $f_{\vtheta}$, i.e., $f^*\in\fF:=\{f_{\vtheta}(\cdot)|\vtheta\in\sR^M\}$. And the training data $S=\{(\vx_i\in\sR^d,y_i=f^*(\vx_i)\in\sR)\}_{i=1}^n$ is sampled from $f^*$. \\
        (iii) The loss function $\ell(\cdot,\cdot)$ is a continuously differentiable distance function and the empirical loss $L_{S}(\vtheta)=\Exp_S\ell(f_{\vtheta}(\vx),y)=\frac{1}{n}\sum_{i=1}^{n}\ell(f(\vx_i;\vtheta),y_i)$.}
\end{assumption*}

Under the above general assumptions, this work focuses on the recovery of target function $f^*$ for the following regression problem when $n\leqslant M$: 
\begin{equation}
    \mathrm{min}_{\vtheta\in\sR^M}\frac{1}{n}\sum_{i=1}^{n}\ell(f_{\vtheta}(\vx_i),f^*(\vx_i)). \label{eq:minF_main}
\end{equation}
We say $f^*$ is successfully recovered when the solution $\vtheta^*$ obtained by an algorithm satisfies $f_{\vtheta^*}=f^*$. For convenience, we define the target set $\fM_{f^*}$ as follows, by which $f^*$ is recovered when $\vtheta^*\in\fM_{f^*}$.

\begin{definition}[target set]\label{def:target}
Suppose we have a model $f_{\vtheta}(\cdot)=f(\cdot;\vtheta):\sR^d\times\sR^M\to \sR$ with model function space $\fF=\{f_{\vtheta}(\cdot)|\vtheta\in\sR^M\}$. For any target function $f'\in\fF$, the target set $\fM_{f'}:=\{\vtheta|f(\cdot;\vtheta)=f'\}$.
\end{definition}

We note that analyzing the recovery of $f^*$ at overparameterization through global training presents significant difficulties due to two primary challenges: (i) the optimization challenge, where it remains uncertain whether the training process can avoid local minima and saddle points to converge to a global minimum \citep{sun2020global}; (ii) the challenge of infinite solutions, where the target set $\mathcal{M}_{f^*}$ is embedded within an approximately $(M-n)$-dimensional manifold of global minima with complex geometry \citep{cooper2018loss,zhang2023structure}, making it difficult to ascertain if global training can precisely reach $\mathcal{M}_{f^*}$. To bypass these obstacles, we introduce a more theoretically tractable variant of recovery as follows.
\begin{definition}[local linear recovery (LLR) guarantee]\label{def:LLR-guarantee}
    Suppose we have a differentiable model $f_{\vtheta}(\cdot)$ with $M$ parameters,  a target function $f^*\in \fF:=\{f_{\vtheta}(\cdot)\}$ and training data $S=\{(\vx_i, f^*(\vx_i))\}_{i=1}^{n}$. \\
    (a) \textbf{LLR-guarantee:} We say $f^*$ has local linear recovery (LLR) guarantee (by model $f_{\vtheta}$ from $S$) if the following condition holds: There exists $\vtheta'\in\fM_{f^*}$ such that
\begin{equation}
    f^*=\mathrm{argmin}_{g\in\widetilde{\fT}_{\vtheta'}}\frac{1}{n}\sum_{i=1}^{n}\ell(g(\vx_i),f^*(\vx_i)), \label{eq:minF1}
\end{equation}
where $\widetilde{\fT}_{\vtheta'}=\{f(\cdot;\vtheta')+\va^\TT\nabla_{\vtheta}f(\cdot;\vtheta')|\va\in\sR^M\}$ is the tangent function hyperplane at $\vtheta'$. If the above condition holds, we say $f^*$ has LLR-guarantee at $\vtheta'$ (by model $f_{\vtheta}$ from $S$). \\
(b) \textbf{$n$-sample LLR-guarantee:} We say a function $f^*\in\fF$ has $n$-sample LLR-guarantee if there exists a $n$-sample dataset $S=\{(\vx_i\in\sR^d,f^*(\vx_i)\in\sR)\}_{i=1}^n$ such that $f^*$ has LLR-guarantee. Furthermore, if $n<M$, then we say $f^*$ has LLR-guarantee at overparameterization.\\
(c) \textbf{$n$-sample LLR-guarantee a.e.:} We say a function $f^*\in\fF$ has $n$-sample LLR-guarantee a.e. (almost everywhere) if the following condition holds: For inputs $X=[\vx_1,\cdots,\vx_n]\in \sR^{d\times n}$ a.e. with respect to $\fL^{d\times m}$  Lebesgue measure, $f^*$ has LLR-guarantee from $S=\{(\vx_i, f^*(\vx_i))\}_{i=1}^{n}$.
\end{definition}

\begin{remark} [LLR-guarantee vs. LLR-guarantee a.e.]
    (i) Although $n$-sample LLR-guarantee seems to be significantly weaker than $n$-sample LLR-guarantee a.e., we will show later that any $n$-sample LLR-guarantee automatically upgrades to $n$-sample LLR-guarantee a.e. for analytic models, e.g., neural networks with tanh, sigmoid or GELU activation. For the pursuit of generality, we mainly focus on the $n$-sample LLR-guarantee in this work. (ii) In general, $n$-sample LLR-guarantee informs the potential of recovering the target from $n$ samples in practice, whereas no $n$-sample LLR-guarantee is a strong indication that $n$ samples are not sufficient to recover the target.
\end{remark}

\begin{prop}\label{prop:greater_n-sample LLR}
     Suppose we have a differentiable model $f_{\vtheta}(\cdot)$ with $M$ parameters. For any target function $f^*\in \fF$, if it has $n$-sample LLR-guarantee, then it has $n'$-sample LLR-guarantee for any $n'\geqslant n$.
\end{prop}

\begin{proof}
    When $f^*$ has $n$-sample LLR-guarantee, there exists $S=\{(\vx_i,f^*(\vx_i))\}_{i=1}^{n}$ and $\vtheta^*\in\fM_{f^*}$ such that the solution set 
$$\phi_{S}^*:=\mathrm{argmin}_{f\in\widetilde{\fT}_{\vtheta^*}}\mathbb{E}_{(\vx,y)\in S}\ell(f(\vx),y)=\{f^*\}.$$
For $S'=S\cup \{(\vx'_i,f^*(\vx'_i))\}_{i=n+1}^{n'}$ with any $\vx'_i\in\sR^M$ for $i\in [n+1:n']$, let $$\phi_{S'}^*:=\mathrm{argmin}_{f\in\widetilde{\fT}_{\vtheta^*}}\mathbb{E}_{(\vx,y)\in S'}\ell(f(\vx),y).$$ 
Obvious, $f^*\in\phi_{S'}^*$, by which the empirical loss attains $0$ on $S'$. Then, any global minimizer $f\in\phi_{S'}^*$ also attains $0$ empirical loss on $S'$, thus $0$ empirical loss on $S$. Therefore,  $\phi_{S'}^*\subset \phi_{S}^*=\{f^*\}$, which yields $$\phi_{S'}^*= \phi_{S}^*=\{f^*\}.$$ 
By Definition \ref{def:LLR-guarantee}, $f^*$ has $n'$-sample LLR-guarantee.
\end{proof}

Proposition. \ref{prop:greater_n-sample LLR} signifies the importance of understanding the minimum sample size with LLR-guarantee for a target function, which is rigorously defined below as an optimistic sample size.
\begin{definition}[optimistic sample size]\label{def:opt_sample_size}
     Suppose we have a differentiable model $f_{\vtheta}(\cdot)$. For any function $f^*\in \fF$, if $f^*$ has LLR-guarantee from $n$ samples but not $n-1$ samples, then its optimistic sample size $$O_{f_{\vtheta}}(f^*)=n.$$
\end{definition}

\subsection{General LLR theory}
In this section, we present the theoretical results of LLR for the regression problem of general differentiable models. In particular, we establish the quantitative relation between the LLR-guarantee and the model rank, by which estimating the optimistic sample size of a target converts to an estimation of its model rank.

\begin{definition}[model rank]
Given any differentiable (in parameters) model $f_{\vtheta}$, the model rank for any $\vtheta^*\in \sR^M$ is defined as
\begin{equation}
\rrank_{f_{\vtheta}}(\vtheta^*):=\rdim\left(\rspan\left\{\partial_{\theta_i} f(\cdot;\vtheta^*)\right\}_{i=1}^M\right),
\end{equation}
where $\rspan\left\{ \phi_i(\cdot)\right\}_{i=1}^M=\{\sum_{i=1}^M a_i\phi_i(\cdot)|a_1,\cdots,a_M\in\sR\}$ and $\rdim(\cdot)$ returns the dimension of a linear function space.
\end{definition}

\begin{definition}[empirical tangent matrix and empirical model rank]\label{def:emp_mr}
Given any differentiable model $f_{\vtheta}$ and training data $S=\{(\vx_i,y_i)\}_{i=1}^n$, at any parameter point $\vtheta^*$,  $\nabla_{\vtheta}f(\mX;\vtheta^*)=[\nabla_{\vtheta}f(\vx_1;\vtheta^*),\cdots, \nabla_{\vtheta}f(\vx_n;\vtheta^*)]$ is referred to as the empirical tangent matrix. Then the empirical model rank is defined as follows
$$\rrank_{S}(\vtheta^*)=\rrank(\nabla_{\vtheta}f(\mX;\vtheta^*)).$$
\end{definition}

\begin{lemma}[LLR condition]\label{lem:lin_stab_c}
$f^*$ has local linear recovery guarantee at $\vtheta^*$ by model $f_{\vtheta}$ from $S$ if and only if $\rrank_{S}(\vtheta^*)=\rrank_{f_{\vtheta}}(\vtheta^*)$.
\end{lemma}

\begin{proof}
Define 
\[
    \Tilde{R}_S(\va) = \frac{1}{n}\sum_{i=1}^n \ell(f(\vx_i, \vtheta^*) + \va^\TT \nabla_{\vtheta}f(\vx_i;\vtheta^*), f^*(\vx_i)). 
\]
First assume that $f^*$ has LLR guarantee at $\vtheta^*$ by model $f_{\vtheta}$ from $S$. Then $f(\cdot; \vtheta^*) = f^*$ by Definition \ref{def:LLR-guarantee}, and this function is unique. This uniqueness implies that for any $\va \in \ker \nabla_{\vtheta} f(\vX; \vtheta^*)$ we must have 
\[
    f(\cdot; \vtheta^*) + \va^\TT \nabla_{\vtheta} f(\cdot; \vtheta^*) = f^*. 
\]
Equivalently, $\va \in \ker \nabla_{\vtheta} f(\cdot; \vtheta^*)$. This shows $\ker \nabla_{\vtheta} f(\vX; \vtheta^*) \subseteq \ker \nabla_{\vtheta} f(\cdot; \vtheta^*)$. But clearly $\ker \nabla_{\vtheta} f(\cdot; \vtheta^*) \subseteq \ker \nabla_{\vtheta} f(\vX; \vtheta^*)$, so the two kernels are equal. It follows that 
\[
    \rrank_S(\vtheta^*) = \rrank\left( \nabla_{\vtheta} f(\vX; \vtheta^*) \right) = \rrank_{f_{\vtheta}}(\vtheta^*)
\]

Conversely, assume that $\rrank_S(\vtheta^*) = \rrank_{f_{\vtheta}}(\vtheta^*)$. Then, similar as above, we have $\ker \nabla_{\vtheta} f(\cdot; \vtheta^*) = \ker \nabla_{\vtheta} f(\vX; \vtheta^*)$. Because $\ell$ is a distance function, for any $\va \in \sR^M$ with $\Tilde{R}_S(\va) = R_S(\vtheta^*) = 0$, we must have 
\[
    \va^\TT \nabla_{\vtheta} f(\vx_i; \vtheta^*) = 0, \quad \forall\, 1\leqslant i \leqslant n. 
\]
Equivalently, $\va \in \ker \nabla_{\vtheta} f(\vX; \vtheta^*)$ and thus $\va^\TT f(\cdot; \vtheta^*) = 0$. This shows 
\[
    f^* = f(\cdot; \vtheta^*) = \mathrm{argmin}_{g\in\widetilde{\fT}_{\vtheta'}}\frac{1}{n}\sum_{i=1}^{n}\ell(g(\vx_i),f^*(\vx_i))
\]
is well-defined, i.e., $f^*$ has LLR guarantee at $\vtheta^*$. 

\end{proof}

\begin{cor}[phase transition of LLR-guarantee at a target point]\label{cor:critical_rank_theta}
    For any $\vtheta'\in\fM_{f^*}$, if training data size $n<\rrank_{f_{\vtheta}}(\vtheta')$, $f^*$ has no local linear recovery guarantee at $\vtheta'$. Otherwise, if $n\geqslant\rrank_{f_{\vtheta}}(\vtheta')$, $f^*$ has $n$-sample LLR-guarantee, i.e., there exits an $n$-sample dataset $S'=\{(\vx_i,f^*(\vx_i))\}_{i=1}^{n}$ such that $f^*$ has local linear recovery guarantee at $\vtheta'$.
\end{cor}

\begin{proof}
    For $n<\rrank_{f_{\vtheta}}(\vtheta')$, $\rrank_{S}(\vtheta')\leqslant n<\rrank_{f_{\vtheta}}(\vtheta')$. By Lemma \ref{lem:lin_stab_c}, $f^*$ has no local linear recovery guarantee at $\vtheta'$. On the other hand, for $n\geqslant\rrank_{f_{\vtheta}}(\vtheta')$, we claim that there exist $\vX = (\vx_1, ...,\vx_n)$ with $\rrank\left( \nabla_{\vtheta} f(\vX; \vtheta^*) \right) = \rrank_{f_{\vtheta}}(\cdot; \vtheta')$. Suppose this is not true. Let 
    \[
        M' = \max_{\vX \in \sR^{d\times n}} \rrank \left(\nabla_{\vtheta} f(\vX; \vtheta')\right)
    \]
    and let $\vx_1, ..., \vx_{M'}$ be such that 
    \[
        \rrank \left[ \nabla_{\vtheta} f((\vx_1,..., \vx_{M'}); \vtheta') \right] = M'. 
    \]
    So in particular $M' < \rrank_{f_{\vtheta}}(\cdot; \vtheta') \leqslant M$. Given $a \in \sR^M$ such that $\nabla_{\vtheta} f(\vx_i; \vtheta) = 0$ for all $1 \le i \le M'$, we must have $a \in \ker \nabla_{\vtheta} f(\vtheta'; \cdot)$. Therefore, 
    \[
        \ker \nabla_{\vtheta} f(\vtheta'; \cdot) \subseteq \ker \nabla_{\vtheta} f((\vx_1,..., \vx_{M'}); \vtheta') \subseteq \ker \nabla_{\vtheta} f(\vtheta'; \cdot). 
    \]
    This means 
    \[
        \dim \ker \nabla_{\vtheta} f(\vtheta'; \cdot) = \dim \ker \nabla_{\vtheta} f(\vX; \vtheta') = M-M'
    \]
    and thus $\rrank_{f_{\vtheta}}(\vtheta') = M'$, a contradiction. So there must exist $\vX = (\vx_1, ...,\vx_n)$ with $\rrank\left( \nabla_{\vtheta} f(\vX; \vtheta^*) \right) = \rrank_{f_{\vtheta}}(\cdot; \vtheta')$. By Lemma \ref{lem:lin_stab_c}, $f^*$ has $n$-sample guarantee. 

\end{proof}

We address above the LLR-guarantee at a target point. In the following, we further address the LLR-guarantee for recovering a target function.

\begin{definition}[model rank for function]\label{def:mr_f}
The model rank for any function $f^*\in\fF$ is defined as 
\begin{equation}
    \rrank_{f_{\vtheta}}(f^*):=\min_{\vtheta'\in\fM_{f^*}}\rrank_{f_{\vtheta}}(\vtheta').
\end{equation}
\end{definition}

\begin{thm}[optimistic sample size estimate]\label{thm:opt_sample_size}
Suppose we have a differentiable model $f_{\vtheta}$. For any function $f^*\in\fF$, we have $$O_{f_{\vtheta}}(f^*)=\rrank_{f_{\vtheta}}(f^*).$$
\end{thm}

\begin{proof}
    When the sample size $n<\rrank_{f_{\vtheta}}(f^*)$,  for any $\vtheta'\in\fM_{f^*}$, $n<\rrank_{f_{\vtheta}}(\vtheta')$. By Corollary \ref{cor:critical_rank_theta}, $f^*$ has no LLR-guarantee at $\vtheta'$. Therefore, $f^*$ has no $n$-sample LLR-guarantee. When the sample size $n\geqslant\rrank_{f_{\vtheta}}(f^*)$, there exists $\vtheta' \in\fM_{f^*}$, such that $n\geqslant\rrank_{f_{\vtheta}}(f^*)=\rrank_{f_{\vtheta}}(\vtheta')$.  By Corollary \ref{cor:critical_rank_theta}, $f^*$ has $n$-sample LLR-guarantee. By Definition \ref{def:opt_sample_size}, we have  $$O_{f_{\vtheta}}(f^*)=\rrank_{f_{\vtheta}}(f^*).$$
\end{proof}

The subsequent two corollaries provide a rigorous formulation of two intuitive assertions about the best-possible cases of regression problem \eqref{eq:minF_main}: (i) $M$ samples are sufficient for recovery; (ii) linear models generally cannot be recovered at overparameterization.

\begin{cor}[generic upper bound for optimistic sample size]\label{cor:generic_ub}
    Suppose we have a differentiable model $f_{\vtheta}$. For any function $f^*\in\fF$, we have $$O_{f_{\vtheta}}(f^*)\leqslant M.$$
\end{cor}

\begin{proof}
    Because $\rrank_{f_{\vtheta}}(\vtheta)\leqslant M$ for any $\vtheta\in\sR^M$, $\rrank_{f_{\vtheta}}(f^*)\leqslant M$ for any $f^*\in\fF$. Therefore, by Theorem \ref{thm:opt_sample_size}, $O_{f_{\vtheta}}(f^*)=\rrank_{f_{\vtheta}}(f^*)\leqslant M$.
\end{proof}

\begin{cor}[no LLR-guarantee for linear models at overparameterization]
    For a linear model $f_{\vtheta}(\vx)=\sum_{i=1}^{M} \theta_i \phi_i(\vx)$ with $\vtheta=[\theta_1,\cdots,\theta_M]^\TT$, if its basis functions are linearly independent, i.e., $\rdim(\rspan\{\phi_i(\cdot)\}_{i=1}^{M})=M$, then $$O_{f_{\vtheta}}(f^*)\equiv M,$$ i.e., none of the functions in the model function space $\fF$ has LLR-guarantee at overparameterization.
\end{cor}

\begin{proof}
    For any $\vtheta^*\in\sR^M$, we have
    \begin{equation*}
        \rrank_{f_{\vtheta}}(\vtheta^*)=\rdim\left(\rspan\left\{\partial_{\theta_i} f(\cdot;\vtheta^*)\right\}_{i=1}^M\right)
        =\rdim\left(\rspan\left\{\phi_i(\vx)\right\}_{i=1}^M\right)
        =M.
    \end{equation*}
    Therefore, $O_{f_{\vtheta}}(f^*)=\rrank_{f_{\vtheta}}(f^*)= M$ for any $f^*\in\fF$.
\end{proof}

\subsection{LLR-guarantee for analytic models}
In this subsection, we exploit the characteristics of analytic functions to enhance the LLR guarantee, extending it to an almost everywhere (a.e.) guarantee for analytic models.
\begin{lemma}\label{lem:zero-measure}
Given $m$ linearly independent analytic functions $\phi_1(\vx),\cdots,\phi_m(\vx)$ with $\phi_i:\sR^d\to\sR$ for all $i\in[m]$, $\rrank(\mPhi(\mX))= m$ almost everywhere (a.e.) with respect to   $\fL^{d\times m}$  Lebesgue measure, where 
    \begin{equation*}
    \mPhi(\mX) := \left[\begin{matrix}
             \phi_1(\vx_1) &... &\phi_m(\vx_1) \\ 
             \vdots                                  &\ddots &\vdots \\
             \phi_1(\vx_m) &... &\phi_m(\vx_m) 
             \end{matrix}\right].
    \end{equation*}
\end{lemma}

\begin{proof}
Clearly, $\det(\mPhi(\cdot)):\sR^{d\times m}\to\sR$ is an analytic function over $\sR^{d\times m}$. In addition, because $\{\phi_i\}_{i=1}^m$ are linearly independent, there exists $\mX\in\sR^{d\times m}$ such that $\det(\mPhi(\cdot))\neq 0$, i.e., $\det(\mPhi(\cdot))$ is not constant zero. By the property of real analytic function \citep{Mityagin2020}, $\rrank(\mPhi(\mX))= m$ a.e. with respect to $\fL^{d\times m}$  Lebesgue measure.
\end{proof}

\begin{cor}\label{cor:zero-measure}
Given $m$ analytic functions $\phi_1(\vx),\cdots,\phi_m(\vx)$ with $\phi_i:\sR^d\to\sR$ for all $i\in[m]$ and $\rdim(\rspan(\{\phi_i(\cdot)\}_{i=1}^m))=r$, $\rrank(\mPhi(\mX))= \min\{n,r\}$ a.e. with respect to $\fL^{d\times n}$ Lebesgue measure.
\end{cor}
\begin{proof}
It is obvious that $\rrank(\mPhi(\mX))\leqslant \min\{n,r\}$.
For $n\leqslant r$, we can always pick $n$ independent functions from $\{\phi_i(\cdot)\}_{i=1}^m$. By Lemma \ref{lem:zero-measure}, $\mPhi(\mX)$ has a rank-$n$ submatrix of $\mPhi(\mX)$ a.e. with respect to Lebesgue measure.
For $n>r$, we have that the submatrix of the first $r$ rows of $\mPhi(\mX)$ has rank $r$ a.e. by Lemma \ref{lem:zero-measure}. Therefore, $\rrank(\mPhi(\mX))= \min\{n,r\}$ a.e. with respect to $\fL^{d\times n}$ Lebesgue measure.
\end{proof}

\begin{thm}[LLR-guarantee a.e. for analytic models]\label{appthm:phase-transition}
Given any analytic model $f_{\vtheta}$ (w.r.t. both inputs and parameters), if target function $f^*\in\fF$ has $n$-sample LLR-guarantee, then $f^*$ has $n$-sample LLR-guarantee a.e.
\end{thm}

\begin{proof}
    If target function $f^*\in\fF$ has $n$-sample LLR-guarantee, by Definition \ref{def:opt_sample_size}, $n\geqslant O_{f_{\vtheta}}(f^*)$. In addition, there exists $\vtheta^*\in\fM_{f^*}$ such that $\rrank_{f_{\vtheta}}(\vtheta^*)=O_{f_{\vtheta}}(f^*)$. Then by Corollary \ref{cor:zero-measure}, for $\mX=[\vx_1,\cdots,\vx_n]$ a.e. with respect to $\fL^{d\times n}$ Lebesgue measure with dataset $S=\{(\vx_i,f^*(\vx_i))\}_{i=1}^{n}$, $\rrank_S(\vtheta^*)=\min\{\rrank_{f_{\vtheta}}(\vtheta^*),n\}=\rrank_{f_{\vtheta}}(\vtheta^*)$. By the LLR condition Lemma \ref{lem:lin_stab_c}, $f^*$ has $n$-sample LLR-guarantee a.e.
\end{proof}

\section{Upper bounds of optimistic sample sizes for DNNs\label{appsec:DNN}}
Building on the LLR theory, the task of estimating optimistic sample sizes for target functions necessitates the elucidation of the model rank across the space of functions. This task is notably challenging, particularly for DNNs with three or more layers, due to the complexity inherent in disentangling the linear dependencies among the compositional tangent functions. Fortunately, we observe that the critical embedding operators introduced in Refs. \citep{zhang2021embedding,zhang2022embedding} can preserve the tangent function space and the model rank when transitioning from a parameter point in a narrower DNN to one in a wider DNN. This observation allows us to constrain the model rank of a target function within a wide DNN by referencing the smallest DNN capable of representing it regardless of the depth. The formal mathematical exposition of these results is provided below.

\subsection{DNNs and the Embedding Principle}

In this subsection, we briefly recapitulate the key elements of Embedding Principle in Refs. \citep{zhang2021embedding,zhang2022embedding}.
Consider $L$-layer ($L\geq 2$) fully-connected DNNs with a general differentiable activation function. We regard the input as the $0$-th layer and the output as the $L$-th layer. Let $m_l$ be the number of neurons in the $l$-th layer. In particular, $m_0=d$ and $m_L=d'$. For any $i,k\in \sN$ and $i<k$, we denote $[i:k]=\{i,i+1,\ldots,k\}$. In particular, we denote $[k]:=\{1,2,\ldots,k\}$.
Given weights $W^{[l]}\in \sR^{m_l\times m_{l-1}}$ and bias $b^{[l]}\in\sR^{m_{l}}$ for $l\in[L]$, we define the collection of parameters $\vtheta$ as a $2L$-tuple (an ordered list of $2L$ elements) whose elements are matrices or vectors
\begin{equation}
    \vtheta=\Big(\vtheta|_1,\cdots,\vtheta|_L\Big)=\Big(\mW^{[1]},\vb^{[1]},\ldots,\mW^{[L]},\vb^{[L]}\Big).
\end{equation}
where the $l$-th layer parameters of $\vtheta$ is the ordered pair $\vtheta|_{l}=\Big(\mW^{[l]},\vb^{[l]}\Big),\quad l\in[L]$.
We may misuse the notation and identify $\vtheta$ with its vectorization $\mathrm{vec}(\vtheta)\in \sR^M$ with $M=\sum_{l=0}^{L-1}(m_l+1) m_{l+1}$.

Given $\vtheta\in \sR^M$, the neural network function $\vf_{\vtheta}(\cdot)$ is defined recursively. First, we write $\vf^{[0]}_{\vtheta}(\vx)=\vx$ for all $\vx\in\sR^d$. Then for $l\in[L-1]$, $\vf^{[l]}_{\vtheta}$ is defined recursively as 
$\vf^{[l]}_{\vtheta}(\vx)=\sigma (\mW^{[l]} \vf^{[l-1]}_{\vtheta}(\vx)+\vb^{[l]})$.
Finally, we denote
\begin{equation}
    \vf_{\vtheta}(\vx)=\vf(\vx,\vtheta)=\vf^{[L]}_{\vtheta}(\vx)=\mW^{[L]} \vf^{[L-1]}_{\vtheta}(\vx)+\vb^{[L]}.
\end{equation}
For notational simplicity, we may drop the subscript $\vtheta$ in $\vf^{[l]}_{\vtheta}$, $l\in[0:L]$.

We formally define the notion of wider/narrower as follows.
\begin{definition}[\textbf{Wider/narrower DNN}]\label{def:narr_wide}
    We write $\mathrm{NN}(\{m_l\}_{l=0}^{L})$ for a fully-connected neural network with width $(m_0,\ldots,m_L)$.
    Given two $L$-layer ($L\geq 2$) fully-connected neural networks $\mathrm{NN}(\{m_l\}_{l=0}^{L})$ and $\mathrm{NN}'(\{m'_l\}_{l=0}^{L})$, if $m'_0=m_0$, $m'_L=m_L$, and for any $l\in[L-1]$, $m'_l\geq m_l$ and $K=\sum_{l=1}^{L-1}(m'_l-m_l)>0$, then we say that $\mathrm{NN}'(\{m'_l\}_{l=0}^{L})$ is wider or $K$-neuron wider than $\mathrm{NN}(\{m_l\}_{l=0}^{L})$ and $\mathrm{NN}(\{m_l\}_{l=0}^{L})$ is narrower or $K$-neuron narrower than $\mathrm{NN}'(\{m'_l\}_{l=0}^{L})$. 
\end{definition}

\begin{thm}[Embedding Principle, Theorem 4.2 in Ref.  \citep{zhang2022embedding}]\label{thm:embeddingPrinciple}
Given any NN and any  $K$-neuron wider NN, there exists a $K$-step composition embedding $\fP$ satisfying  that:
For any given data $S$, loss function $\ell(\cdot,\cdot)$, activation function $\sigma(\cdot)$,  given any   critical point $\vtheta^{\rc}_{\rnarr}$ of the narrower NN,   $\vtheta^{\rc}_{\rwide}:=\fP(\vtheta^{\rc}_{\rnarr})$ is still a critical point of the  $K$-neuron wider NN with the same output function, i.e., $\vf_{\vtheta^{\rc}_{\rnarr}}=\vf_{\vtheta^{\rc}_{\rwide}}$.
\end{thm}

\begin{figure}
    \centering
    \includegraphics[width=0.8\textwidth]{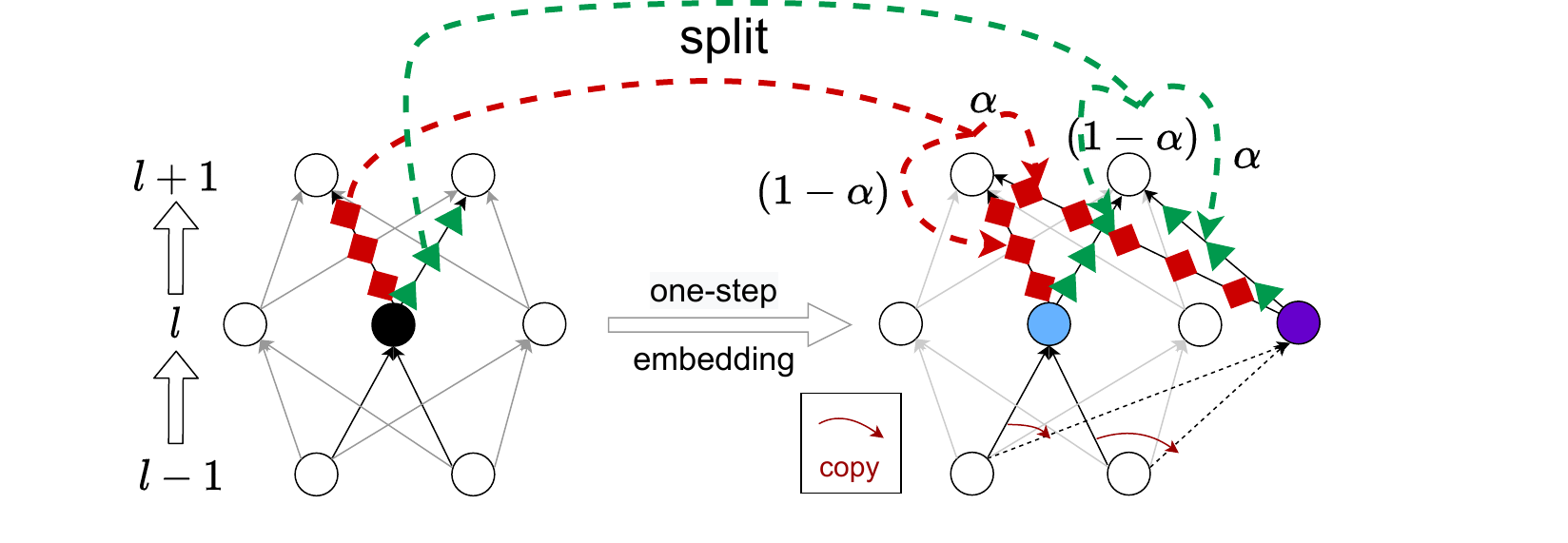} 
    \label{fig:onestep}
\caption{(Figure 2 in \cite{zhang2021embedding}) Illustration of one-step splitting embedding. The black neuron in the left network is split into the blue and purple neurons in the right network. The red (green) output weight of the black neuron in the left net is splitted into two red (green) weights in the right net with ratio $\alpha$ and $(1-\alpha)$, respectively.    \label{fig:onestep}}
\end{figure}

The cornerstone of the Embedding Principle's proof lies in constructing a series of critical embeddings that map any given DNN to a wider one. Figure \ref{fig:onestep} illustrates a typical type of critical embedding, specifically the 'splitting' embedding. This process involves dividing a single neuron into two, while simultaneously preserving both the neural network's output function and its criticality. It is readily apparent that this splitting embedding can be seamlessly adapted to Convolutional Neural Networks (CNNs) and Residual Networks (ResNets), suggesting that the Embedding Principle is applicable to these architectures as well. Consequently, it is reasonable to infer that all subsequent results that depend on the presence of critical embeddings for fully-connected neural networks can be extended to CNNs and ResNets. Additionally, \cite{zhang2022embedding} describes other varieties of critical embeddings, such as the null embedding and the general compatible embedding.

\subsection{Upper bounding optimistic sample size via critical mappings\label{appsec:ub_estimate}}

In this subsection, we first provide a general definition of critical mappings, by which the previously proposed critical embeddings (see \cite{zhang2022embedding} Definition 4.2 for details) are special cases. Then, we prove Lemma \ref{lem:ub_rank}, showing that uncovering critical mappings is an important means for obtaining an upper bound estimate of the optimistic sample size. This general result combined with the embedding principle of DNNs directly provides an upper bound estimate of optimistic sample sizes for general DNNs in Theorem \ref{thm:upper_rank_DNN}. We also illustrate the upper bounds for a general depth-$L$ DNN without bias terms in Table. \ref{tab:partial_rank}.

\begin{definition}[critical mapping]
Given differentiable model A $f_{\vtheta_A}=f(\cdot;\vtheta_A)$ with $\vtheta_A\in\sR^{M_A}$ and differentiable model B $g_{\vtheta_B}=g(\cdot;\vtheta_B)$ with $\vtheta_B\in\sR^{M_B}$, $\fP:\sR^{M_A}\to\sR^{M_B}$ is a critical mapping from model $A$ to $B$ if given any $\vtheta\in \sR^{M_A}$, we have\\
(i) output preserving: $f_{\vtheta}=g_{\fP(\vtheta)}$;\\
(ii) criticality preserving: for any data $S=\{(\vx_i,y_i)\}_{i=1}^n$ with mean-squared error (MSE) empirical risk function $R_S(h)=\frac{1}{2}\sum_{i=1}^n(h(\vx_i)-y_i)^2$, if $\nabla_{\vtheta}R_S(f_{\vtheta})=\vzero$, then $\nabla_{\vtheta}R_S(g_{\fP(\vtheta)})=\vzero$.
\end{definition}

\begin{lemma}[upper bound of optimistic sample size]\label{lem:ub_rank}
Given two models $f_{\vtheta_A}=f(\cdot;\vtheta_A)$ with $\vtheta_A\in\sR^{M_A}$ and $g_{\vtheta_B}=g(\cdot;\vtheta_B)$ with $\vtheta_B\in\sR^{M_B}$, if there exists a critical mapping $\fP$ from model $A$ to $B$, then the optimistic sample size $O_g(f^*)\leqslant O_f(f^*)\leqslant M_A$ for any $f^*\in\fF_{A}$.
\end{lemma}
\begin{rmk}
If $M_B\gg M_A$, this upper bound estimate is highly informative, indicating target recovery capability at heavy overparameterization for model $B$.  Importantly, this lemma establishes the relation between our rank stratification and previous studies about the critical embedding for the DNN loss landscape analysis. As a result, the critical embedding intrinsic to the DNN architecture not only benefits optimization as studied in previous works, but also profoundly benefits the recovery/generalization performance.
\end{rmk}
\begin{proof}
 By Definition \ref{def:mr_f}, for any $f^*\in\fF_{f}$, there exists $\vtheta^*\in\sR^{M_A}$ such that $\rrank_{f_{\vtheta}}(\vtheta^*)=\rrank_{f_{\vtheta}}(f^*)$. Then, $g_{\fP(\vtheta^*)}=f_{\vtheta^*}=f^*$. Because $R_S(\vtheta)=\frac{1}{2}\sum_{i=1}^n(f_{\vtheta}(\vx_i)-y_i)^2$, we have 
 \[
    \nabla_{\vtheta}R_S(f_{\vtheta^*})=\sum_{i=1}^n (y_i-f^*(\vx_i))\nabla_{\vtheta^*}f_{\vtheta^*}(\vx_i)
 \]
 and $\nabla_{\fP(\vtheta)}R_S(g_{\fP(\vtheta^*)})=\sum_{i=1}^n (y_i-f^*(\vx_i))\nabla_{\fP(\vtheta)}g_{\fP(\vtheta^*)}(\vx_i)$. Because $\fP$ is criticality preserving for arbitrary data $S$, we have $\ker(\nabla_{\vtheta}f_{\vtheta^*}(\mX))\subseteq\ker(\nabla_{\fP(\vtheta)}g_{\fP(\vtheta^*)}(\mX))$ for any $\mX:=[\vx_1,\cdots,\vx_n]$. Here, $\nabla_{\vtheta}f_{\vtheta^*}(\mX)=[\nabla_{\vtheta}f_{\vtheta^*}(\vx_1),\cdots,\nabla_{\vtheta}f_{\vtheta^*}(\vx_n)]$. Because $\rrank_S(\fP(\vtheta^*))+\rdim(\ker(\nabla_{\fP(\vtheta)}g_{\fP(\vtheta^*)}(\mX)))=\rrank_S(\vtheta^*)+\rdim(\ker(\nabla_{\vtheta}f_{\vtheta^*}))=n$, we have $\rrank_S(\fP(\vtheta^*))\leqslant \rrank_S(\vtheta^*)$ for any data $S$ (See Definition \ref{def:emp_mr} for $\rrank_{S}(\cdot)$). Taking the infinite data limit, we obtain $\rrank_g(\fP(\vtheta^*))\leqslant\rrank_f(\vtheta^*)\leqslant M_A$. Therefore, $\rrank_g(f^*)\leqslant\rrank_f(f^*)\leqslant M_A$ for any $f^*\in\fF_{f}$. By Theorem \ref{thm:opt_sample_size}, $O_g(f^*)\leqslant O_f(f^*)\leqslant M_A$.
\end{proof}

 As a direct consequence of Lemma \ref{lem:ub_rank} and Theorem \ref{thm:embeddingPrinciple}, we obtain the following theorem.
\begin{thm}[upper bound of optimistic sample size for DNNs]\label{thm:upper_rank_DNN}
Given any NN with $M_{\rwide}$ parameters, for any function in the function space of a narrower NN with $M_{\rnarr}$ parameters and for any $f^*\in\fF_{\rnarr}$, we have $O_{f_{\vtheta_{\rwide}}}(f^*)\leqslant O_{f_{\vtheta_{\rnarr}}}(f^*)\leqslant M_{\rnarr}$.
\label{thm:app_embedding}
\end{thm}

\begin{proof}
    By Theorem \ref{thm:embeddingPrinciple}, there exists a critical mapping $\fP:\sR^{M_\rnarr}\to\sR^{M_\rwide}$ from the narrower NN to the given NN. Then, by Lemma \ref{lem:ub_rank}, $O_{f_{\vtheta_{\rwide}}}(f^*)\leqslant O_{f_{\vtheta_{\rnarr}}}(f^*)\leqslant M_{\rnarr}$.
\end{proof}

It should be noted that while Theorem \ref{thm:upper_rank_DNN} is proven here for fully-connected DNNs, the results are readily extendable to CNNs and ResNets by employing their respective critical embeddings, such as the splitting embedding. Theorem \ref{thm:upper_rank_DNN} provides a significantly more precise upper bound for optimistic sample sizes than the generic upper bound presented in Corollary \ref{cor:generic_ub}, as demonstrated for an $L$-layer DNN in Table \ref{tab:partial_rank}. Furthermore, from this theorem, we can infer two corollaries that tackle two fundamental issues in the theory of deep learning: (i) the guarantee of LLR at overparameterization for DNNs; and (ii) the benefits derived from the layer-wise architectural design of DNN models.

\begin{table}[htb]
\centering
\renewcommand{\arraystretch}{2.0} 
\begin{tabular}{|c|c|c|}
\hline
\multicolumn{3}{|c|}{
model:                                                                  $f_{\boldsymbol{\theta}}(\boldsymbol{x}) = \boldsymbol{W}^{[L]}\sigma(\cdots\sigma(\boldsymbol{W}^{[1]}\boldsymbol{x})\cdots)$, $\mW^{[l]}=\sR^{m_l\times m_{l-1}}$,$m_L=1$, $m_0=d$}   \\ \hline
$f^*$                                                                  &    generic bound (Cor. \ref{cor:generic_ub})      & our bound (Thm. \ref{thm:upper_rank_DNN})                                                                                                                                                   \\ \hline
$\mathcal{F}_{\{1, 1, \cdots, 1\}}$                             & $M$                              & $d+L-1$                                                                                                                                                                   \\ \hline
$\vdots$                                                        & \vdots                             & $\vdots$                                                                                                                                                                  \\ \hline
$\mathcal{F}_{\{m_i'\}_{i=1}^{L-1}}, 1\leq m_i' \leq m_i$         & $M$                          & $dm_1'+m_1'm_2'+\cdots+m_{L-2}'m_{L-1}' + m_{L-1}'$                                                                                                                 \\ \hline
$\vdots$                                                       & \vdots                               & $\vdots$                                                                                                                                                                  \\ \hline
$\mathcal{F}_{\{m_i\}_{i=1}^{L-1}}$                           & $M$                                & $M$                                                                                                                             \\ \hline
\end{tabular}
\vspace{5pt}
\caption{Upper bound of optimistic sample size for a general $L$-layer fully-connected DNN with width-${\{m_i\}_{i=1}^{L-1}}$. For the simplicity of presentation, we consider the DNN without bias terms. Its total number of parameters $M=dm_1+m_1m_2+\cdots+m_{L-2}m_{L-1} + m_{L-1}$. $\fF_{\{m_i\}_{i=1}^{L-1}}$ denotes the function space of the $L$-layer DNN with width-$\{m_i\}_{i=1}^{L-1}$ for hidden layers. \label{tab:partial_rank}}
\end{table}

\begin{cor}[LLR-guarantee at overparameterization for DNNs]\label{cor:LLR-op-DNN}
    For a DNN model, all functions expressible by any narrower DNNs have LLR-guarantee at overparameterization.
\end{cor}

\begin{proof}
    We denote $f_{\vtheta_\rwide}$ as the DNN model and $M_{\rwide}$ as its total parameter size.  For any narrower DNN $f_{\vtheta_\rnarr}$ (see Definition \ref{def:narr_wide}), its parameter size $M_{\rnarr}<M_{\rwide}$. For any $f^*\in\fF_{\rnarr}$, by Theorem \ref{thm:upper_rank_DNN}, the optimistic sample size $O_{f_{\vtheta_{\rwide}}}(f^*)\leqslant O_{f_{\vtheta_{\rnarr}}}(f^*)\leqslant M_{\rnarr}<M_{\rwide}$. By the Definition \ref{def:opt_sample_size} of the optimistic sample size and Definition \ref{def:LLR-guarantee} of the LLR-guarantee, $f^*$ has LLR-guarantee at overparameterization.
\end{proof}

As far as we know, this is the first recovery guarantee at overparametrization for general DNNs. Though LLR-guarantee is a relatively weak form of recovery guarantee, it sets a solid ground for further improving the guarantee to stronger types of recovery, e.g., local or even global recovery. In particular, we reasonably conjecture that having LLR-guarantee at overparameterization is a necessary condition for having any stronger type of recovery guarantee. 

\begin{cor}[free expressiveness in width]\label{cor:free_expr}
    The optimistic sample size of a target function expressible by any DNN never increases as the DNN gets wider.
\end{cor}

\begin{proof}
    For any DNN $f_{\vtheta_\rnarr}$, any wider DNN $f_{\vtheta_\rwide}$ and a target function $f^*\in\fF_{\rnarr}$, by Theorem \ref{thm:upper_rank_DNN}, the optimistic sample size $O_{f_{\vtheta_{\rwide}}}(f^*)\leqslant O_{f_{\vtheta_{\rnarr}}}(f^*)$. 
\end{proof}

Corollary \ref{cor:free_expr} ensures that the potential for recovery in the sense of LLR is not compromised when employing wider neural networks. This provides theoretical justification for the application of large DNNs in modeling even simple target functions.

\section{Optimistic sample sizes for two-layer tanh NNs}\label{sec: Optimistic sample sizes for two-layer NN}

Theorem \ref{thm:upper_rank_DNN} prompts the investigation of two pertinent questions: (i) the tightness of the established upper bound, and (ii) the methodology for determining the precise value of the optimistic sample size. In the subsequent analysis, we specifically address these issues for two-layer NNs with tanh activation functions. We provide exact estimates of the optimistic sample sizes that reach the upper limits as delineated by Theorem \ref{thm:upper_rank_DNN}, with the results detailed in Table \ref{table:FFN}. Furthermore, we extend our estimation to the optimistic sample sizes for two-layer tanh-CNNs, both with and without the implementation of weight-sharing. A comparative analysis of the optimistic sample sizes for fully-connected NNs versus CNNs reveals that superfluous connections among neurons lead to an increase in the optimistic sample size, which in turn adversely affects the fitting performance in terms of LLR.

\begin{table}[htb]
\centering
\renewcommand{\arraystretch}{2.} 
\begin{tabular}{|c|c|c|c|}
\hline
\multicolumn{4}{|c|}{model: $f_{\boldsymbol{\theta}}(\boldsymbol{x}) = \sum_{i=1}^{m}a_i\tanh(\boldsymbol{w}_i^\TT \boldsymbol{x}), \boldsymbol{x}\in \mathbb{R}^d, \boldsymbol{\theta} = (a_i, \boldsymbol{w}_i)_{i=1}^m$}                                         \\ \hline
$f^*$ & generic bound (Cor. \ref{cor:generic_ub}) & upper bound (Thm. \ref{thm:upper_rank_DNN})&  $O_{f_{\vtheta}}(f^*)$ (Thm. \ref{thm:opt_tanhFNN})                        \\ \hline
$\{0(\cdot)\}$& $0$ & $0$ & $0$                            \\ \hline
$\mathcal{F}^{\mathrm{NN}}_1\backslash \{0(\cdot)\}$&$m(d+1)$&$d+1$&$d+1$                                          \\ \hline
$\vdots$ &$\vdots$ &$\vdots$ &  $\vdots$                                                       \\ \hline
\makecell[c]{$\mathcal{F}^{\mathrm{NN}}_k\backslash \mathcal{F}^{\mathrm{NN}}_{k-1}$}& $m(d+1)$ & $k(d+1)$& $k(d+1)$      \\ \hline
$\vdots$&$\vdots$ &$\vdots$ & $\vdots$                                                  \\ \hline
\makecell[c]{$\mathcal{F}^{\mathrm{NN}}_m\backslash \mathcal{F}^{\mathrm{NN}}_{m-1}$}& $m(d+1)$& $m(d+1)$& $m(d+1)$ \\ \hline
\end{tabular}
\vspace{5pt}
\caption{Results of the optimistic sample sizes for two-layer width-$m$ tanh-NN. Here $\mathcal{F}^{\mathrm{NN}}_k:=\{\sum_{i=1}^k a^*_i\sigma(\boldsymbol{w}_i^{*\TT}\boldsymbol{x})| a_i\in\sR,\vw_i\in\sR^d\}$  denotes the function space of the width-$k$ tanh-NN. \label{table:FFN}}
\end{table}

\subsection{Theoretical preparation\label{appsec:two-layer_thm}}

The exact estimation of model rank hinges on unraveling the linear dependencies among tangent functions. The subsequent results concerning linear independence form the bedrock for estimating the model rank in two-layer neural networks.

\begin{prop}[linear independence of neurons]\label{prop: linear independence of neurons}
Let $\sigma: \sR \to \sR$ be any analytic function such that $\sigma^{(n_j)}(0) \neq 0$ for an infinite sequence of distinct indices $\{n_j\}_{j=1}^\infty$. Given $d \in \sN$ and $m$ distinct weights $\vw_1, ..., \vw_m \in \sR^d \backslash \{\vzero\}$, such that $\vw_k\neq\pm \vw_j$ for all $1 \leqslant k < j \leqslant m$. Then $\{\sigma(\vw_i^\TT \vx), \sigma'(\vw_i^\TT \vx)x_1, ..., \sigma'(\vw_i^\TT \vx)x_d\}_{i=1}^m$ is a linearly independent function set.  
\end{prop}
\begin{proof}
For $x$ sufficiently close to $0 \in \sR$, we can write $\sigma(x) = \sum_{j=0}^\infty c_j x^j$, where $c_j = \sigma^{(j)}(0)/(j!)$. Then, $\sigma'(x) = \sum_{j=1}^\infty j c_j x^{j-1}$ . Suppose that the set is not linearly independent. Choose not-all-zero constants $\{\alpha_i\}_{i=1}^m$ and $\{\beta_{i1}, ..., \beta_{id}\}_{i=1}^m$ such that 
\begin{equation*}
    \vx \mapsto \sum_{i=1}^m \left( \alpha_i \sigma(\vw_i^\TT \vx) + \sum_{t=1}^d \beta_{it} \sigma'(\vw_i^\TT \vx)x_t \right) 
\end{equation*} 
is a zero map on $\sR^d$, where $x_t$ denotes the $t$-th component of input. For $k, j, i \in [d]$, define the sets 
\begin{align*}
    A_{k,j} &:= \{\vx \in \sR^d | \left< \vx, \vw_k \pm \vw_j \right> = 0\} \\ 
    B_i     &:= \{\vx \in \sR^d | \left<\vx, \vw_i \right> = 0\}. 
\end{align*} 
Clearly, each $A_{k,j}$ is the union of two linear subspaces of dimension $(d-1)$, while each $B_i$ is a possibly empty affine subspace of dimension $(d-1)$. Thus, 
\begin{equation*}
    E := \left( \cup_{1\leqslant k,j \leqslant m} A_{k,j} \right) \cup \left( \cup_{i=1}^d B_i \right)
\end{equation*} 
has $\fL^{d}$ Lebesgue measure zero. Let $\ve \in \sR^d \backslash E$. Denote $p_i := \left< \vw_i, \ve \right>$ for each $i \in [m]$. Since $p_i \neq p_j$ and $p_i + p_j \neq 0$ whenever $i \neq j$, we can, without loss of generality, assume that $|p_1| > |p_2| > ... > |p_m|>0$. For any sufficiently small $\vep$ and any $i, t$ we have 
\begin{align*}
    \sigma(\vw_i^\TT (\vep \ve)) &= \sum_{j=0}^\infty (c_j p_i^j) \vep^j, \\ 
    \sigma'(\vw_i^\TT (\vep\ve))(\vep\ve)_t &= e_t \sum_{j=1}^\infty (j c_j p_i^{j-1}) \vep^j.  
\end{align*}
Thus, for sufficiently small $\vep$, 
\begin{equation}\label{linear independence of neurons eq1}
\begin{aligned}
    \sum_{i=1}^m \left( \alpha_i \sigma(\vw_i^\TT(\vep\ve)) + \sum_{t=1}^d \beta_{it} \sigma'(\vw_i^\TT (\vep\ve))(\vep\ve)_t\right) 
    &= \left(\sum_{i=1}^m \alpha_i \right) c_0 + \sum_{j=1}^\infty c_j \sum_{i=1}^m \left( \alpha_i + \frac{1}{p_i}\sum_{t=1}^d j\beta_{it} e_t \right) p_i^j \vep^j \\
    &= 0.
\end{aligned}
\end{equation}
We have $c_j \sum_{i=1}^m \left( \alpha_i + \frac{1}{p_i}\sum_{t=1}^d j\beta_{it} e_t \right) p_i^j = 0$ for all $j \in \sN$. In particular, for any $j \geq 2$, since $n_j \geq 1$ and $c_{n_j} \neq 0$, we have $\sum_{i=1}^m \left( \alpha_i + \frac{1}{p_i} \sum_{t=1}^d n_j\beta_{it} e_t \right) p_i^{n_j} = 0$, which yields 
\begin{equation*}
    \alpha_1 + \frac{1}{p_1}\sum_{t=1}^d n_j\beta_{1t} e_t = - \sum_{i=2}^m \left( \alpha_i + \frac{1}{p_i}\sum_{t=1}^d n_j\beta_{it} e_t \right) \frac{p_i^{n_j}}{p_1^{n_j}}. 
\end{equation*}
If $m = 1$, by taking limits $j\to \infty$, we have $\alpha_1 = \sum_{t=1}^d \beta_{1t}e_t = 0$.

Otherwise, since $|p_1| > |p_i|$ for any $2 \leqslant i \leqslant m$, it follows that, by taking limits $j\to \infty$, 
\begin{equation*}
     \lim_{j\rightarrow \infty}\left(\alpha_1 + \frac{1}{p_1}\sum_{t=1}^d n_j \beta_{1t} e_t \right) = \lim_{j\rightarrow \infty}- \sum_{i=2}^m \left( \alpha_i + \frac{1}{p_i}\sum_{t=1}^d n_j\beta_{it} e_t \right) \frac{p_i^{n_j}}{p_1^{n_j}} = 0. 
\end{equation*}
Thus, we also have $\alpha_1 = \sum_{t=1}^d \beta_{1t}e_t = 0$. For $m > 2$, we may rewrite Eq. \eqref{linear independence of neurons eq1} as 
\begin{equation*}
    \alpha_2 +\frac{1}{p_2} \sum_{t=1}^d n_j \beta_{2t} e_t = - \sum_{i=3}^m \left( \alpha_i + \frac{1}{p_i}\sum_{t=1}^d n_j \beta_{it} e_t \right) \frac{p_i^{n_j}}{p_2^{n_j}}
\end{equation*} 
for each $j \geq 2$, and take limits as we do above to deduce that $\alpha_2 + \frac{1}{p_2}\sum_{t=1}^d n_j \beta_{2t} e_t = 0$. By repeating this procedure for  at most $m$ times, we conclude that $\alpha_i + \frac{1}{p_i}\sum_{t=1}^d n_j \beta_{it} e_t = 0$ for all $i \in [m]$. Then, $\alpha_i = \sum_{t=1}^d \beta_{it}e_t = 0$ for any $i\in [m]$. For each $i$, $\sum_{t=1}^d \beta_{it}e_t'$ is a linear function of $\ve'$ on the open set $\sR^d\backslash E$ which vanishes on a neighborhood of $\ve$, we must have $\alpha_i = \beta_{it} = 0$ for any $i \in [m], t \in [d]$. Therefore, $\{\sigma(\vw_i^\TT \vx), \sigma'(\vw_i^\TT \vx)x_1, ..., \sigma'(\vw_i^\TT \vx)x_d\}_{i=1}^m$ must be a linearly independent set. 
\end{proof}

\begin{cor}[model rank estimate for two-layer NNs]\label{cor:mr_theta_2layerFNN}
Let $\sigma = \tanh$. Given $d \in \sN$, weights $\vw_1, ..., \vw_m \in \sR^d$, $a_1, ..., a_m \in \sR$, we have 
\begin{equation*}
    \rdim(\rspan \{\sigma(\vw_i^{\TT} \vx), a_i \sigma'(\vw_i^{\TT} \vx)x_1, ..., a_i \sigma'(\vw_i^{\TT} \vx)x_d \}_{i=1}^m)=m_{\vw} + m_{a} d,
\end{equation*} 
where $m_{\vw}=\frac{1}{2}|\{\vw_i,-\vw_i|\vw_i\neq \vzero, i\in[m]\}|$ indicating the number of independent neurons, $m_{a}=\frac{1}{2}|\{\vw_i,-\vw_i|\vw_i\neq \vzero, a_i\neq 0, i\in[m]\}| + |\{\vw_i|\vw_i= \vzero, a_i\neq 0, i\in[m]\}|$ indicating the number of independent effective neurons. Here, $|\cdot|$ is the cardinality  of a set, i.e., number of different elements in a set. 
\label{cor:dimension}
\end{cor}
\begin{proof}
Note that $\sigma = \tanh$ is analytic and $\sigma^{(2n+1)}(0) \neq 0$ for all $n$. Because $\tanh$ is an odd function, we have $\tanh(x) = -\tanh(-x)$ and $\tanh(0) = 0$. Therefore, given $\vw_i, \vw_j \neq \vzero$ with $\vw_i = \pm \vw_j$, $\rspan \{\sigma(\vw_i^\TT x)\} = \rspan\{\sigma(\vw_j^\TT x)\}$ and $\rspan \{\sigma'(\vw_i^\TT x)x_1, ..., \sigma'(\vw_i^\TT x)x_d\} = \rspan \{\sigma'(\vw_j^\TT x)x_1, ..., \sigma'(\vw_j^\TT x)x_d\}$. Since there are $m_{\vw}$ different non-zero weights, by Proposition~\ref{prop: linear independence of neurons} we have 
\begin{equation*}
    \dim \left( \rspan \{\sigma(\vw_i^\TT x)\}_{i=1}^m \right)= m_{\vw}. 
\end{equation*}
Furthermore, note that 
\begin{equation*}
    \rspan \{a_i \sigma'(\vw_i^\TT x)x_1, ..., a_i \sigma'(\vw_i^\TT x)x_d \}_{i=1}^m = \rspan \{\sigma'(\vw_i^\TT x)x_1, ..., \sigma'(\vw_i^\TT x)x_d: a_i \neq 0, i \in [m]\}. 
\end{equation*}
Thus, by Proposition~\ref{prop: linear independence of neurons}, 
\begin{align*}
    &\,\,\,\,\,\,\,\dim \left( \rspan \{\sigma'(\vw_i^\TT x)x_1, ..., \sigma'(\vw_i^\TT x)x_d: \vw_i \neq \vzero, a_i \neq 0, i \in [m]\} \right) \\ 
    &= \frac{1}{2}|\{\vw_i,-\vw_i|\vw_i\neq \vzero, a_i\neq 0, i\in[m]\}| \cdot d. 
\end{align*}
Now suppose that $\vw_j = \vzero \in \sR^d$ for some $j \in [m]$. Since $\sigma'(\vw^\TT x) = \sigma'(0) \neq 0$ for all $x \in \sR^d$, 
\begin{equation*}
    \rspan \{\sigma'(\vw_j^\TT x)x_1, ..., \sigma'(\vw_j^\TT x)x_d\} = \rspan \{x_1, ..., x_d\} 
\end{equation*} 
which consists only of linear functions. By the nonlinearity of $\tanh$, we conclude that 
\begin{equation*}
    \dim \left( \rspan \{\sigma'(\vw_i^\TT x)x_1, ..., \sigma'(\vw_i^\TT x)x_d\} \right) = m_a d 
\end{equation*} 
and thus $\dim \left( \rspan \{\sigma(\vw_i^\TT x), \sigma'(\vw_i^\TT x)x_1, ..., \sigma'(\vw_i^\TT x)x_d \} \right) = m_{\vw} + m_a d$ as desired. 
\end{proof}

The above corollary directly gives rise to the following result.
\begin{cor}
\label{cor:app_DNN}
Let $\sigma = \tanh$. Given distinct weights $\vw_1, ..., \vw_m \in \sR^d\backslash\{\boldsymbol{0}\}$ satisfying $\vw_k\neq\pm \vw_j$ for $k\neq j$ , and $a_1, ..., a_m \in \sR\backslash\{0\}$, we have 
\begin{equation*}
    \rdim(\rspan \{\sigma(\vw_i^{\TT} \vx), a_i \sigma'(\vw_i^{\TT} \vx)x_1, ..., a_i \sigma'(\vw_i^{\TT} \vx)x_d \}_{i=1}^m)=m(d+1).
\end{equation*}
\end{cor}

Next we consider the estimate of the model rank for CNNs which are widely used in practice. Here we consider the case where the input has two-dimensional indices, which is the most general case for the image input. The following two propositions can be directly generalized to the model rank estimate of CNNs with an input of one index dimension. 

\begin{definition}[ineffective neurons/kernels]\label{def: ineffective neurons/kernels}
    For two-layer tanh NNs, we say a neuron/kernel is output-ineffective if its output weight array is zero but input weight array is nonzero. We say a neuron/kernel is input-ineffective if its input weight array is zero but output weight array is nonzero. We say a neuron/kernel is null if both input and out weight arrays are zero. All these neurons/kernels are ineffective. Neurons/kernels that are not ineffective are effective.
\end{definition}

\begin{remark}
    Estimating the model rank at parameter points with input-ineffective neurons is complicated for CNNs. Luckily, we notice that model rank of a target function $f^*\in\fF$ can be obtained by considering only $\vtheta'\in\fM_{f^*}$  with no input-ineffective neurons. This is because any $\vtheta'\in\fM_{f^*}$ with input-ineffective neurons associates to a $\vtheta''\in\fM_{f^*}$ such that (i) input-ineffective neurons at $\vtheta'$ are replaced by null neurons (ii) $\rrank_{f_{\vtheta}}(\vtheta'')\leqslant \rrank_{f_{\vtheta}}(\vtheta')$. Therefore, we only estimate the model rank for CNNs at the parameter points with no input-ineffective neurons in the following propositions.
\end{remark}

\begin{prop}[model rank estimate for CNNs (with weight sharing)]\label{prop:CNN_ws}
    Given $m \in \sN$, $d \in \sN$ and $s \in [d]$. For any $l \in [m]$, let $\mK_l$ be a $(s \times s)$ matrix. Consider CNNs with stride = 1. For a tanh-CNN $f_{\vtheta}$ with weight sharing, 
    \begin{equation*}
        f_{\vtheta} (\mI) = \sum_{l=1}^m \sum_{i,j=1}^{d+1-s} a_{lij} \tanh \left(\sum_{\alpha, \beta} I_{i+s-\alpha, j+s-\beta} K_{l;\alpha, \beta}\right), \mI \in \sR^{d\times d}, 
    \end{equation*}
    its model rank at $\vtheta = (a_{lij}, \mK_{l})_{l, i, j}$ with no input-ineffective neurons is given by
    $$\rrank_{f_{\vtheta}}(\vtheta)=m_a s^2 + m_K(d+1-s)^2,$$
 where $m_{K}=\frac{1}{2}|\{\mK_l,-\mK_l|l\in[m],\mK_l\neq \vzero\}|$ indicating the number of independent kernels, $m_{a}=\frac{1}{2}\sum_{\mK\in \fK}\operatorname{dim}(\rspan\{a_{l,:,:}\}_{l\in h(\mK)})$ indicating the number of independent effective neurons. Here $\fK = \{\mK_l, -\mK_l|l\in [m],\mK_l\neq \vzero\}$, $h$ is a function over $\fK$ s.t. for each $\mK\in \fK, h(\mK) = \{l|l\in [m], \mK_l = \pm\mK\}$. $|\cdot|$ is the cardinality  of a set, i.e., number of different elements in a set, and $a_{l,:,:}$ denotes the $(d+1-s) \times (d+1-s)$ matrix whose entries are $a_{lij}$'s.
\end{prop}

\begin{proof}
    Let $\sigma = \tanh$. We first consider the case in which there is no ineffective neuron (i.e., $a_{lij} \ne 0$ for all $l, i, j$) and $\mK_l \pm \mK_{l'} \neq \mzero$ for any distinct $l,l' \in [m]$. In this case the model rank is the dimension of the following function space (with respect to variable $\mI\in\sR^{d\times d}$) 
\begin{align*}
    &\,\,\,\,\,\,\,\rspan \left\{\frac{\partial f_{\vtheta}}{\partial a_{lij}} , \frac{\partial f_{\vtheta}}{\partial K_{l; \alpha, \beta}} \right\} \\ 
    &= \rspan \left\{ \sigma \left(\sum_{\alpha', \beta'} I_{i+s-\alpha', j+s-\beta'}K_{l;\alpha', \beta'} \right), \right. \\ 
    &\,\,\,\,\,\,\,\left.\sum_{i',j'=1}^{d+1-s} a_{li'j'} \sigma'\left(\sum_{\alpha', \beta'} I_{i'+s-\alpha', j'+s-\beta'}K_{l;\alpha', \beta'}\right) I_{i'+s-\alpha, j'+s-\beta} \right\}_{l, i, j,\alpha, \beta}, 
\end{align*}
where $l \in [m]$ and $\alpha, \beta \in [s]$. Next, we prove by contradiction that the set of functions 
\begin{equation*}
    \left\{ \sum_{i,j=1}^{d+1-s} a_{lij} \sigma' \left(\sum_{\alpha', \beta'} I_{i+s-\alpha', j+s-\beta'}K_{l;\alpha', \beta'} \right) I_{i+s-\alpha, j+s-\beta} \right\}_{l \in [m], \alpha,\beta \in [s]}
\end{equation*} 
are linearly independent.  If they are not linearly independent, there exist not all zero constants $\zeta_{l11}, ..., \zeta_{lss}$ for $l\in[m]$, such that 
\begin{equation*}
    \sum_{l=1}^m \sum_{\alpha, \beta=1}^s \zeta_{l\alpha\beta} \sum_{i,j=1}^{d+1-s} a_{lij} \sigma' \left(\sum_{\alpha', \beta'} I_{i+s-\alpha', j+s-\beta'}K_{l;\alpha', \beta'} \right) I_{i+s-\alpha, j+s-\beta}= 0, 
\end{equation*} 
which implies that the set of functions 
\begin{equation*}
    \left\{ a_{lij} \sigma'\left( \sum_{\alpha', \beta'} I_{i+s-\alpha', j+s-\beta'}K_{l;\alpha', \beta'}\right) I_{i+s-\alpha, j+s-\beta}\right\}_{l,i,j,\alpha,\beta}
\end{equation*} 
are linearly dependent, contradicting Proposition \ref{prop: linear independence of neurons}. 

In presence of null neurons and output-ineffective neurons, if $a_{lij} = 0$ for certain $l \in [m]$ and all $i,j \in \{1, ..., d+1-s\}$, 
\begin{equation*}
    \sum_{i,j=1}^{d+1-s} a_{lij} \sigma' \left(\sum_{\alpha', \beta'} I_{i+s-\alpha', j+s-\beta'}K_{l;\alpha', \beta'} \right) I_{i+s-\alpha, j+s-\beta}= 0
\end{equation*} 
for all $\alpha, \beta \in [s]$. If $\mK_{l} = \vzero$ for certain $l \in [m]$, then by Definition \ref{def: ineffective neurons/kernels} $a_{lij} = 0$ must hold for all $i, j$. Thus, we have 
$$ \sigma \left(\sum_{\alpha', \beta'} I_{i+s-\alpha', j+s-\beta'}K_{l;\alpha', \beta'} \right) I_{i+s-\alpha, j+s-\beta}= 0,$$
$$\sum_{i,j=1}^{d+1-s} a_{lij} \sigma' \left(\sum_{\alpha', \beta'} I_{i+s-\alpha', j+s-\beta'}K_{l;\alpha', \beta'} \right) I_{i+s-\alpha, j+s-\beta}= 0.$$

Moreover, notice that two kernels with $\mK_l = \pm\mK_{l'}$ can be reduced to one while maintaining model rank if and only if the corresponding output weights $a_{l,:,:}$ and $a_{l',:,:}$ are linearly dependent. Then, similar to Corollary~\ref{prop: linear independence of neurons}, we conclude that the model rank is $m_a s^2 + m_K(d+1-s)^2$. 
\end{proof}

\begin{prop}[model rank estimate for CNNs without weight sharing]\label{prop:CNN_ns}
Given $m \in \sN$, $d \in \sN$ and $s \in [d]$. For any $l \in [m]$ and $i,j \in [d+1-s]$, let $\mK_{lij}$ be a $(s \times s)$ matrices. 
Consider CNNs with stride = 1. 
For a $\tanh$ CNN $f_{\vtheta}$ without weight sharing, 
\begin{equation*}
    f_{\vtheta} (\mI) = \sum_{l=1}^m \sum_{i,j=1}^{d+1-s} a_{lij} \tanh\left(\sum_{\alpha, \beta} I_{i+s-\alpha, j+s-\beta}K_{lij; \alpha, \beta}\right), \mI \in \sR^{d\times d}, 
\end{equation*}
its model rank at $\vtheta = (a_{lij}, \mK_{lij})_{l, i, j}$ with no input-ineffective neuron is given by
$$\rrank_{f_{\vtheta}}(\vtheta)=m_a s^2 + m_K,$$
where $m_{K}=\frac{1}{2}|\{p(\mK_{lij}),-p(\mK_{lij})|l\in[m], i,j\in [d+1-s],\mK_{lij}\neq\vzero\}|$ indicating the number of independent kernels, $m_{a}=\frac{1}{2}|\{p(\mK_{lij}),-p(\mK_{lij})|l\in[m], i,j\in [d+1-s], a_{lij}\neq 0\}|$ indicating the number of independent effective neurons. Here $p$ is the padding function over kernels, i.e., for each ($s\times s$) kernel $\mK_{lij}$, $p(\mK_{lij}) \in \sR^{d\times d}$ s.t. $p(\mK_{lij})[i:i+s-1, j:j+s-1] = \mK_{lij}$ and the other elements of $p(\mK_{lij})$ are zero. $|\cdot|$ is the cardinality  of a set, i.e., number of different elements in a set. 
\end{prop}

\begin{proof}
Let $\sigma=\tanh$. The model rank is the dimension of the following function space 
\begin{align*}
    &\,\,\,\,\,\,\,\rspan \left\{\frac{\partial f_{\vtheta}}{\partial a_{lij}} , \frac{\partial f_{\vtheta}}{\partial K_{lij; \alpha, \beta}} \right\}_{l,i,j,\alpha,\beta} \\ 
    &= \rspan \left\{ \sigma \left(\sum_{\alpha', \beta'} I_{i+s-\alpha', j+s-\beta'} K_{lij;\alpha', \beta'} \right), \right.\\
    &\quad \left.a_{lij} \sigma' \left(\sum_{\alpha', \beta'} I_{i+s-\alpha', j+s-\beta'} K_{lij;\alpha', \beta'} \right) I_{i+s-\alpha, j+s-\beta}\right\}_{l,i,j,\alpha, \beta}, 
\end{align*}
where $l \in [m]$, $1 \leqslant i, j \leqslant d+1-s$, and $\alpha, \beta \in [s]$. Also note that if $a_{lij} = 0$ for some $l \in [m]$ and $i,j \in \{1, ..., d+1-s\}$, then 
\begin{equation*}
    a_{lij} \sigma' \left(\sum_{\alpha', \beta'} I_{i+s-\alpha', j+s-\beta'} K_{lij;\alpha', \beta'} \right) I_{i+s-\alpha, j+s-\beta}= 0
\end{equation*}
for all $\alpha, \beta \in [s]$. If $\mK_{lij} = \vzero$ for some $l \in [m]$ and $i,j \in \{1, ..., d+1-s\}$, because there is no input-ineffective neurons, we must have $a_{lij} = 0$. Then
$$\sigma \left(\sum_{\alpha', \beta'} I_{i+s-\alpha', j+s-\beta'} K_{lij;\alpha', \beta'} \right) I_{i+s-\alpha, j+s-\beta}= 0,$$
$$a_{lij} \sigma' \left(\sum_{\alpha', \beta'} I_{i+s-\alpha', j+s-\beta'} K_{lij;\alpha', \beta'} \right) I_{i+s-\alpha, j+s-\beta}= 0.$$
It follows from Proposition \ref{prop: linear independence of neurons} that this space has dimension $m_a s^2 + m_K$. 
\end{proof}

\subsection{Optimistic sample size estimates\label{appsec:FCNN_estimate}}

\begin{thm}[optimistic sample sizes for two-layer tanh-NN] \label{thm:opt_tanhFNN}
Given a two-layer NN $f_{\boldsymbol{\theta}}(\boldsymbol{x}) = \sum_{i=1}^{m}a_i\tanh(\boldsymbol{w}_i^\TT \boldsymbol{x}), \boldsymbol{x}\in \mathbb{R}^d, \boldsymbol{\theta} = (a_i, \boldsymbol{w}_i)_{i=1}^m,$
for any target function $f^* \in\fF_k^{\mathrm{NN}} \backslash \fF_{k-1}^{\mathrm{NN}}$ with $0\leqslant k\leqslant m$, the optimistic sample size 
$$O_{f_{\vtheta}}(f^*)=k(d+1).$$
Here $\mathcal{F}^{\mathrm{NN}}_k:=\{\sum_{i=1}^k a^*_i\sigma(\boldsymbol{w}_i^{*\TT}\boldsymbol{x})| a_i\in\sR,\vw_i\in\sR^d\}$ for $k\in \sN^+$, $\mathcal{F}^{\mathrm{NN}}_0:=\{0(\cdot)\}$ and $\mathcal{F}^{\mathrm{NN}}_{-1}:=\emptyset$.
\end{thm}

\begin{proof}
Given any target function $f^*\in\fF_k^{\mathrm{NN}} \backslash \fF_{k-1}^{\mathrm{NN}}$, for $k=0$, $f^*=0(\cdot)$ and $O_{f_{\vtheta}}(f^*)=\rrank_{f_{\vtheta}}(f^*)=0=k(d+1)$. For $0<k\leqslant m$, we have
$$\fF_k^{\mathrm{NN}} \backslash \fF_{k-1}^{\mathrm{NN}}=\{\sum_{i=1}^k a_i\tanh(\boldsymbol{w}_i^{\TT}\boldsymbol{x}), a_i\neq0, \vw_i\neq\vzero, \vw_i\neq\pm\vw_j\}.$$ Therefore, there exists $\vtheta^* = (a_i^*, \vw_i^*)_{i=1}^{k}$ with $a_i^*\neq0$, $\vw^*_i\neq\vzero$, and $\vw^*_i\neq\pm\vw^*_j$ for $i\neq j$, such that $f^* = f_{\vtheta^*} := \sum_{i=1}^k a^*_i\tanh(\boldsymbol{w}_i^{*\TT}\boldsymbol{x})$.  By the upper bound estimate Theorem \ref{thm:upper_rank_DNN}, $O_{f_{\vtheta}}(f^*)\leqslant k(d+1)$.

By definition, the model rank of $f^*$ is the minimal model rank among all parameters recovering $f^*$ in the target set $\fM_{f^*}$. For any $\vtheta'=(a_i',\vw_i')_{i=1}^m \in \fM_{f^*}$, by Corollary \ref{cor:mr_theta_2layerFNN}, $\rrank_{f_{\vtheta}}(\vtheta')=m_{\vw}' + m_{a}' d,$
where $m_{\vw}'=\frac{1}{2}|\{\vw_i',-\vw_i'|\vw_i'\neq \vzero, i\in[m]\}|$, $m_{a}=\frac{1}{2}|\{\vw_i',-\vw_i'|\vw_i'\neq \vzero, a_i'\neq 0, i\in[m]\}| + |\{\vw_i'|\vw_i'= \vzero, a_i'\neq 0, i\in[m]\}|$. Because
$$\sum_{i=1}^m a_i'\tanh({\vw_i'}^\TT\vx)=\sum_{i=1}^k a_i^*\tanh(\vw_i^{*\TT}\vx)=f^*(\vx),$$
by the linear independence of neurons Proposition \ref{prop: linear independence of neurons}, $m_{\vw}'\geqslant k$ and $m_{a}'\geqslant k$. Then 
$\rrank_{f_{\vtheta}}(\vtheta')\geqslant k(d+1)$, which yields $O_{f_{\vtheta}}(f^*)=\rrank_{f_{\vtheta}}(f^*)\geqslant k(d+1)$. Therefore $O_{f_{\vtheta}}(f^*)=k(d+1)$.

\end{proof}

\begin{thm}[optimistic sample sizes for two-layer tanh-CNN] \label{thm:Opt_ss_tanhCNN}
Given a $m$-kernel two-layer CNN with weight sharing with $2$-d input $I\in\sR^{d\times d}$, $s\times s$ kernel and stride $1$
\begin{align*}
    f_{\vtheta} (\mI) = \sum_{l=1}^m \sum_{i,j=1}^{d+1-s} a_{lij} \tanh\left(\sum_{\alpha, \beta} I_{i+s-\alpha, j+s-\beta}K_{l; \alpha, \beta}\right), \quad \mI \in \sR^{d\times d},
\end{align*}
for any target function $f^* \in\fF_k^{\mathrm{CNN}} \backslash \fF_{k-1}^{\mathrm{CNN}}$ with $0\leqslant k\leqslant m$, the optimistic sample size 
$$O_{f_{\vtheta}}(f^*)=k(s^2+(d+1-s)^2).$$
Here $\mathcal{F}^{\mathrm{CNN}}_k$ indicates the function space of $k$-kernel CNN for $k\in \sN^+$, $\mathcal{F}^{\mathrm{CNN}}_0:=\{0(\cdot)\}$ and $\mathcal{F}^{\mathrm{CNN}}_{-1}:=\emptyset$.
\end{thm}

\begin{proof}
Let $\sigma = \tanh$. The above theorem obviously holds for $k=0$. For any target function $f^* \in\fF_k^{\mathrm{CNN}} \backslash \fF_{k-1}^{\mathrm{CNN}}$ with $0< k\leqslant m$, there exists $\vtheta^* = (a^*_{lij}, \mK^*_{l})_{l\in[k], i, j\in[d+1-s]}$ satisfying (i) $\mK^*_{l}\neq \pm \mK^*_{l'}$ for any $l\neq l'$ and (ii) $\forall l\in[k],  \exists a^*_{lij}\neq 0$, such that 
$$f^*(\mI) = f_{\vtheta^*}(\mI) = \sum_{l=1}^{k}\sum_{i,j=1}^{d+1-s} a^*_{lij} \sigma\left(\sum_{\alpha, \beta} I_{i+s-\alpha, j+s-\beta}K^*_{l; \alpha, \beta}\right).$$
By the upper bound estimate Theorem \ref{thm:upper_rank_DNN}, $O_{f_{\vtheta}}(f^*)\leqslant k(s^2+(d+1-s)^2)$ \footnote{Although Theorem \ref{thm:upper_rank_DNN} surely holds for CNNs, we do not prove it in our work because it requires proving the Embedding Principle for CNNs out of the focus of the current work. For the rigor of our proof, Theorem \ref{thm:upper_rank_DNN} can be walk around as follows. Considering $\vtheta'=(a_{lij}', \mK_{l}')_{l\in[m],i,j\in[d+1-s]}$ with $a_{lij}'=a^*_{lij}, \mK_{l}'=\mK^*_{l}$ for $l\in[k]$ and $a_{lij}'=0, \mK_{l}'=\vzero$ for $l>k$, then $O_{f_{\vtheta}}(f^*)=\rrank_{f_{\vtheta}}(f^*)\leqslant \rrank_{f_{\vtheta}}(\vtheta')=k(s^2+(d+1-s)^2)$.}.

Next we prove that $k(s^2+(d+1-s)^2)$ is also a lower bound of $\rrank_{f_{\vtheta}}(\vtheta')$ for $\vtheta'\in\fM_{f^*}$. Note that, we only need to consider the parameter points with no input-ineffective neuron. For any $\vtheta'=(a_{lij}', \mK_{l}')_{l\in[m], i, j\in[d+1-s]}\in\fM_{f^*}$ with no input-ineffective neuron, by Proposition \ref{prop:CNN_ns},  
$$\rrank_{f_{\vtheta}}(\vtheta')=m_a' s^2 + m_K'(d+1-s)^2,$$ 
where $m_{K}'=\frac{1}{2}|\fK|$, $m_{a}'=\frac{1}{2}\sum_{\mK\in \fK}\operatorname{dim}(\rspan\{a_{l,:,:}\}_{l\in h(\mK)})$ with $\fK = \{\mK_l', -\mK_l'|l\in [m],\mK_l'\neq\vzero\}$. Here $h$ is a function over $\fK$ such that for each $\mK\in \fK, h(\mK) = \{l|l\in [m], \mK_l' = \pm\mK\}$. Because
$$\sum_{l=1}^{k}\sum_{i,j=1}^{d+1-s} a^*_{lij} \sigma\left(\sum_{\alpha, \beta} I_{i+s-\alpha, j+s-\beta}K^*_{l; \alpha, \beta}\right)=\sum_{l=1}^{m}\sum_{i,j=1}^{d+1-s} a'_{lij} \sigma\left(\sum_{\alpha, \beta} I_{i+s-\alpha, j+s-\beta}K'_{l; \alpha, \beta}\right),$$
by the linear independence of neurons Proposition \ref{prop: linear independence of neurons}, $m_{K}\geqslant k$ and $m_{a}\geqslant k$. Therefore $\rrank_{f_{\vtheta}}(\vtheta')\geqslant k(s^2+(d+1-s)^2)$ for $\vtheta'\in\fM_{f^*}$, which yields $O_{f_{\vtheta}}(f^*)=\rrank_{f_{\vtheta}}(f^*)\geqslant k(s^2+(d+1-s)^2)$. Then we obtain $O_{f_{\vtheta}}(f^*)=k(s^2+(d+1-s)^2)$.
\end{proof}

\begin{thm}[optimistic sample sizes of CNN functions in two-layer CNN without weight sharing] \label{thm:Opt_ss_tanh-ns-CNN}
We consider a $m$-kernel two-layer CNN without weight sharing with $2$-d input $I\in\sR^{d\times d}$, $s\times s$ kernel and stride $1$
\begin{align*}
    f_{\vtheta} (\mI) = \sum_{l=1}^m \sum_{i,j=1}^{d+1-s} a_{lij} \tanh\left(\sum_{\alpha, \beta} I_{i+s-\alpha, j+s-\beta}K_{lij; \alpha, \beta}\right), \quad \mI \in \sR^{d\times d}.
\end{align*}
For any target function expressible by a CNN (with sharing) $f^* \in\fF_k^{\mathrm{CNN}} \backslash \fF_{k-1}^{\mathrm{CNN}}\subset \fF_k^{\mathrm{CNN-NS}}$ with $0\leqslant k\leqslant m$, then the optimistic sample size 
$$O_{f_{\vtheta}}(f^*)=(s^2+1)(k(d+1-s)^2-m_{\mathrm{null}}),$$
where $m_{\mathrm{null}}=|\{(l,i,j)|a_{lij}^*=0\}|$ counts the number of neurons with zero output weight.
\end{thm}

\begin{proof}
Let $\sigma = \tanh$. The above theorem obviously holds for $k=0$. For any target function $f^* \in\fF_k^{\mathrm{CNN}} \backslash \fF_{k-1}^{\mathrm{CNN}}$ with $0< k\leqslant m$, there is a point $\vtheta^*$ of CNN (with sharing) such that  
$$f^*(\mI) = f_{\vtheta^*}(\mI) = \sum_{l=1}^{k}\sum_{i,j=1}^{d+1-s} a^*_{lij} \sigma\left(\sum_{\alpha, \beta} I_{i+s-\alpha, j+s-\beta}H^*_{l; \alpha, \beta}\right),$$
where $\mH^*_{l}\neq \pm \mH^*_{l'}$ for any $l\neq l'$ and $\forall l\in[k],  \exists a^*_{lij}\neq 0$. Then
$$f^*(\mI) = f_{\vtheta^*_{\mathrm{NS}}}(\mI) = \sum_{l=1}^{k}\sum_{i,j=1}^{d+1-s} a^*_{lij} \sigma\left(\sum_{\alpha, \beta} I_{i+s-\alpha, j+s-\beta}K^*_{lij; \alpha, \beta}\right),$$
where $\mK^*_{lij}=\mH^*_{l}$ for all $(l,i,j)\in\{(l,i,j)|a_{lij}^*\neq 0\}$, $\mK^*_{lij}=\vzero$ for $(l,i,j)\in\{(l,i,j)|a_{lij}^*= 0\}$  and $\vtheta^*_{\mathrm{NS}}=(a^*_{lij},\mK^*_{lij})_{l\in[k],i,j\in[d+1-s]}$.
By Proposition \ref{prop:CNN_ws}, the model rank at $\vtheta^*_{\mathrm{NS}}$ for the $k$-kernel no-sharing CNN becomes $$\rrank_{\mathrm{CNN^{\mathrm{NS}}_k}}(\vtheta^*_{\mathrm{NS}})=m_a^* s^2 + m_K^*,$$
where 
\[
    m_{K}^*=\frac{1}{2}|\{p(\mK_{lij}^*),-p(\mK_{lij}^*)|\mK_{lij}^*\neq \vzero,l\in[k], i,j\in [d+1-s]\}|=k(d+1-s)^2-m_{\mathrm{null}}
\]
and 
\[
    m_{a}^*=\frac{1}{2}|\{p(\mK_{lij}^*),-p(\mK_{lij}^*)|l\in[k], i,j\in [d+1-s], a_{lij}\neq 0\}|=k(d+1-s)^2-m_{\mathrm{null}}, 
\]
and $p$ is the padding function over kernels, i.e., for each ($s\times s$) kernel $\mK_{lij}$, $p(\mK_{lij}) \in \sR^{d\times d}$ s.t. $p(\mK_{lij})[i:i+s-1, j:j+s-1] = \mK_{lij}$ and the other elements of $p(\mK_{lij})$ are zero. Similar to the proof of Theorem \ref{thm:Opt_ss_tanhCNN}, by the upper bound estimate Theorem \ref{thm:upper_rank_DNN}, 
$$O_{\mathrm{CNN^{\mathrm{NS}}_m}}(f^*)\leqslant O_{\mathrm{CNN^{\mathrm{NS}}_k}}(f^*) \leqslant \rrank_{\mathrm{CNN^{\mathrm{NS}}_k}}(\vtheta^*_{\mathrm{NS}}) = (s^2+1)(k(d+1-s)^2-m_{\mathrm{null}}).$$ 

Also similar to the proof of Theorem \ref{thm:Opt_ss_tanhCNN}, for any $\vtheta'_{\mathrm{NS}}=(a_{lij}',\mK_{lij}')_{l\in[m],i,j\in[d+1-s]}\in\fM_{f^*}$ with no input-ineffective neuron, by the linear independence of neurons Proposition \ref{prop: linear independence of neurons}, we have 
$$m_{K}'=\frac{1}{2}|\{p(\mK_{lij}'),-p(\mK_{lij}')|l\in[m], i,j\in [d+1-s],\mK_{lij}'\neq\vzero\}|\geqslant k(d+1-s)^2-m_{\mathrm{null}},$$
$$m_{a}'=\frac{1}{2}|\{p(\mK_{lij}'),-p(\mK_{lij}')|l\in[m], i,j\in [d+1-s], a_{lij}'\neq 0\}|\geqslant k(d+1-s)^2-m_{\mathrm{null}}.$$
Therefore $\rrank_{\mathrm{CNN^{\mathrm{NS}}_m}}(\vtheta')\geqslant (s^2+1)(k(d+1-s)^2-m_{\mathrm{null}})$ for $\vtheta'\in\fM_{f^*}$, which yields $O_{\mathrm{CNN^{\mathrm{NS}}_m}}(f^*)=\rrank_{\mathrm{CNN^{\mathrm{NS}}_m}}(f^*)\geqslant (s^2+1)(k(d+1-s)^2-m_{\mathrm{null}})$. Then we obtain $O_{\mathrm{CNN^{\mathrm{NS}}_m}}(f^*)=(s^2+1)(k(d+1-s)^2-m_{\mathrm{null}})$.
\end{proof}

\textbf{Costly expressiveness in connection for two-layer NNs.} Drawing from the aforementioned findings, we present a comparative analysis of optimistic sample sizes across various architectures, including CNNs with and without weight sharing, as well as fully-connected NNs, as depicted in Fig. \ref{fig:comp}. It is important to note that, to ensure an equitable comparison, the total number of hidden neurons is held constant $m(d+1-s)^2$ across the different architectures. As demonstrated in Table \ref{table:model_comparing}, for a typical image data with dimension $d=28$, if the target function is recoverable by a CNN with width $m$, then the model rank for different NN architectures exhibits significant variation, ranging from $685m$ for a CNN with weight sharing, to $6760m$ for a CNN without weight sharing, and up to $530660m$ for a fully-connected NN. This stark contrast underscores the vast disparities in their target recovery performance especially when the training data is limited.

\begin{figure}[htbp]
    \centering
    \includegraphics[width=0.8\textwidth]{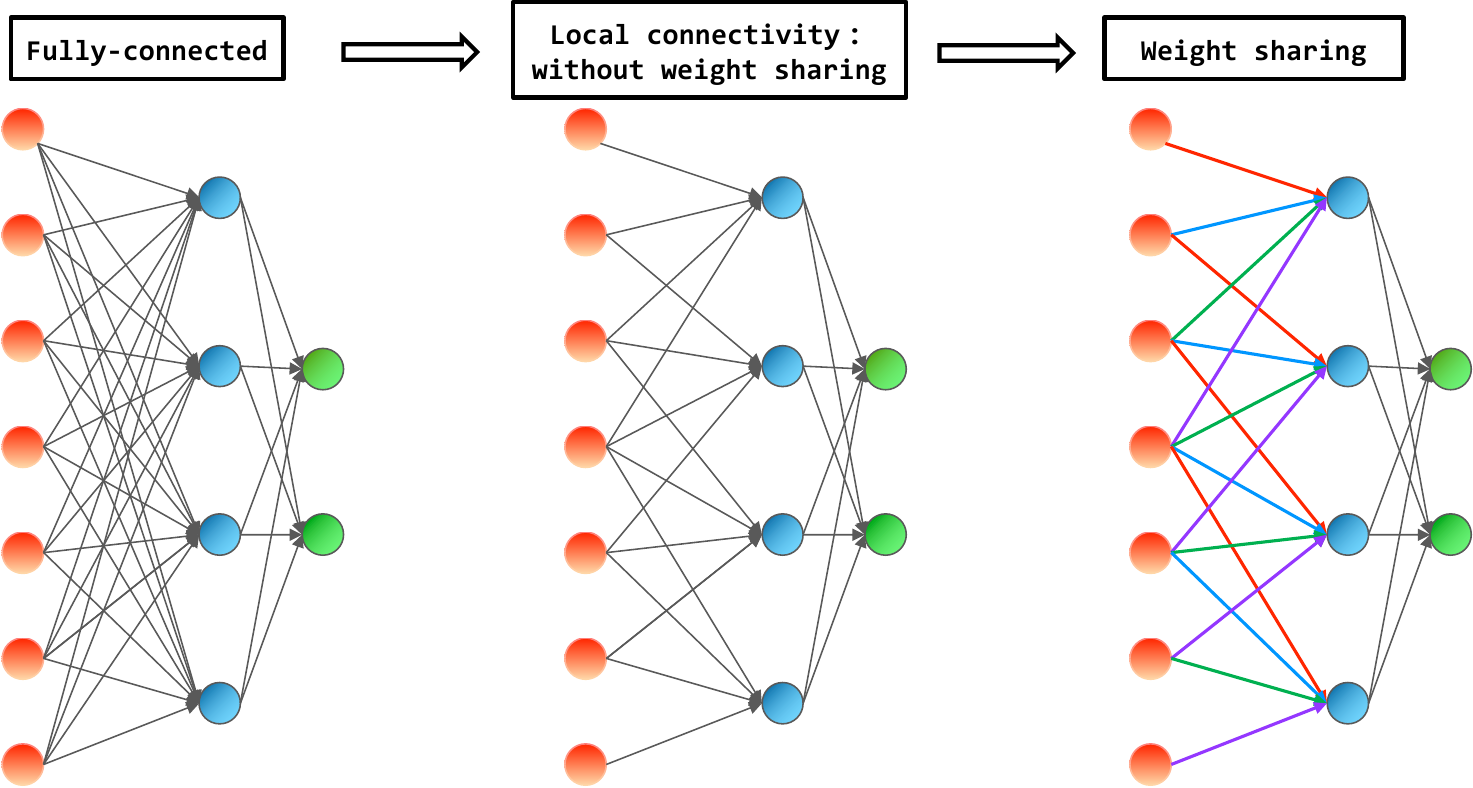}
    \caption{Illustration of architectures from fully-connected NN to CNN for comparison. \label{fig:comp}}
\end{figure}

\begin{table}[htb]
\caption{The optimistic sample size for two-layer tanh-CNN with $m$-kernels of size $s\times s$ and stride $1$. The input $\vx\in\sR^{d\times d}$. For functions in each CNN function set, we also present their model rank in the corresponding CNN without weight sharing and the corresponding fully-connected NN. Here $m_{\mathrm{n}}=m_{\mathrm{null}}=|\{(l,i,j)|a_{lij}^*=0\}|$ counts the number of zero output weights in the target function. \label{table:model_comparing}}
\centering
\renewcommand{\arraystretch}{1.6} 
\small
\begin{tabular}{|c|c|c|c|}
\hline
$f^*$                                                           & CNN & CNN (no sharing)                    & Fully-connected NN                            \\ \hline
$\{0\}$                                                             & $0$                     & $0$                                           & $0$                                           \\ \hline
$\mathcal{F}_1^{\mathrm{CNN}}\backslash\{0\}$                   & $s^2 + (d+1-s)^2$       & $(s^2+1)\left((d+1-s)^2-m_{\mathrm{n}}\right)$  & $(d^2+1)\left((d+1-s)^2-m_{\mathrm{n}}\right)$  \\ \hline
$\vdots$                                                        & $\vdots$                & $\vdots$                                      & $\vdots$                                      \\ \hline
$\mathcal{F}_k^{\mathrm{CNN}}\backslash\mathcal{F}_{k-1}^{\mathrm{CNN}}$ & $k(s^2 + (d+1-s)^2)$    & $(s^2+1)\left(k(d+1-s)^2-m_{\mathrm{n}}\right)$ & $(d^2+1)\left(k(d+1-s)^2-m_{\mathrm{n}}\right)$ \\ \hline
$\vdots$                                                        & $\vdots$                & $\vdots$                                      & $\vdots$                                      \\ \hline
$\mathcal{F}_m^{\mathrm{CNN}}\backslash\mathcal{F}_{m-1}^{\mathrm{CNN}}$ & $m(s^2 + (d+1-s)^2)$    & $(s^2+1)\left(m(d+1-s)^2-m_{\mathrm{n}}\right)$ & $(d^2+1)\left(m(d+1-s)^2-m_{\mathrm{n}}\right)$ \\ \hline
\end{tabular}
\end{table}

\section{Experimental results}
In Fig. \ref{fig: network_stab}, we conduct experiments to assess the practical significance of the previously estimated optimistic sample sizes in relation to the actual fitting performance of two-layer tanh neural networks (NNs) with varying architectures. Our experiments are centered around the target function defined as:
\begin{equation}\label{eq:NN_target}
    f^{*}(\boldsymbol{x}) = \boldsymbol{W}^{*[2]}\tanh(\boldsymbol{W}^{*[1]}\boldsymbol{x}),
\end{equation}
where $\boldsymbol{W}^{*[2]}=[1,1,1]$ and
$$\boldsymbol{W}^{*[1]}=\left[\begin{array}{cccccc}
0.6 &0.8 &1 & 0 &0 \\
0 &0.6 &0.8 &1 & 0 \\
0 &0 &0.6 &0.8 &1 \\
\end{array}\right].$$
We generate both training and test datasets by sampling input data from a standard normal distribution and computing the output using the target function. We employ two-layer tanh-NNs, each with a bias term for the hidden neurons, in a variety of architectures and with different kernels/widths, to fit training datasets of sizes ranging from $1$ to $63$.

It is noteworthy that for a single-kernel CNN, with or without weight sharing, or a fully-connected NN with width $3$ (denoted as $1\textnormal{x}$ in Fig. \ref{fig: network_stab}(b-d)), the model rank coincides with the number of model parameters. Under these conditions, as depicted in Fig. \ref{fig: network_stab}(a), the CNN demonstrates a notably earlier transition to almost-$0$ test error compared to other architectures. However, this finding is somewhat expected since recoveries occur within the traditional over-determined/underparameterized regime.

In Fig. \ref{fig: network_stab}(b-d), we scale up the kernels/widths of the NNs by a factor of $N$, indicated by $N\textnormal{x}$ for each architecture. For $N=100$, the parameter counts for the models are $700$, $1500$, and $2100$, respectively. According to Table \ref{table:model_comparing}, the model rank for the target remains at $7$, $15$, and $21$ (marked by yellow dashed lines), irrespective of the value of $N$. In Fig. \ref{fig: network_stab}(b-d), we observe a postponement in the transition to accurate target recovery for $N>1$, meaning that the test error decreases to nearly zero at a sample size larger than the model rank. Notably, the transition to almost-$0$ test error is much closer to the model rank than to the parameter count, particularly for large $N$ values. We acknowledge that various factors, such as suboptimal hyperparameter tuning, could contribute to the observed delay in recovery. Identifying an optimal training strategy and hyperparameters that enable recovery of a target function as close to its model rank as possible remains a significant challenge in the field.

\begin{figure}[htbp]
	\centering
 	\subfigure[different network types]{\includegraphics[width=0.49\textwidth]{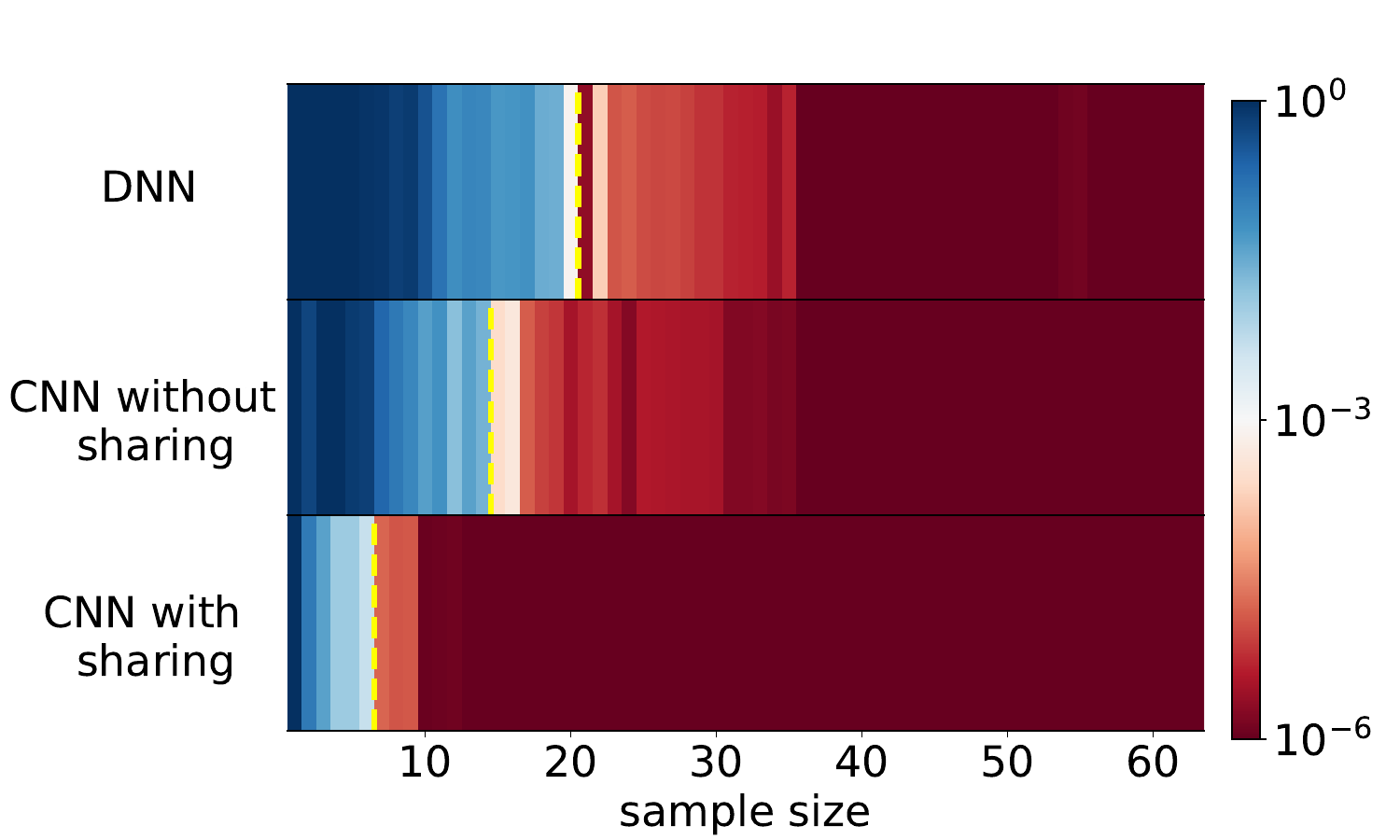}}
	\subfigure[DNN]{\includegraphics[width=0.49\textwidth]{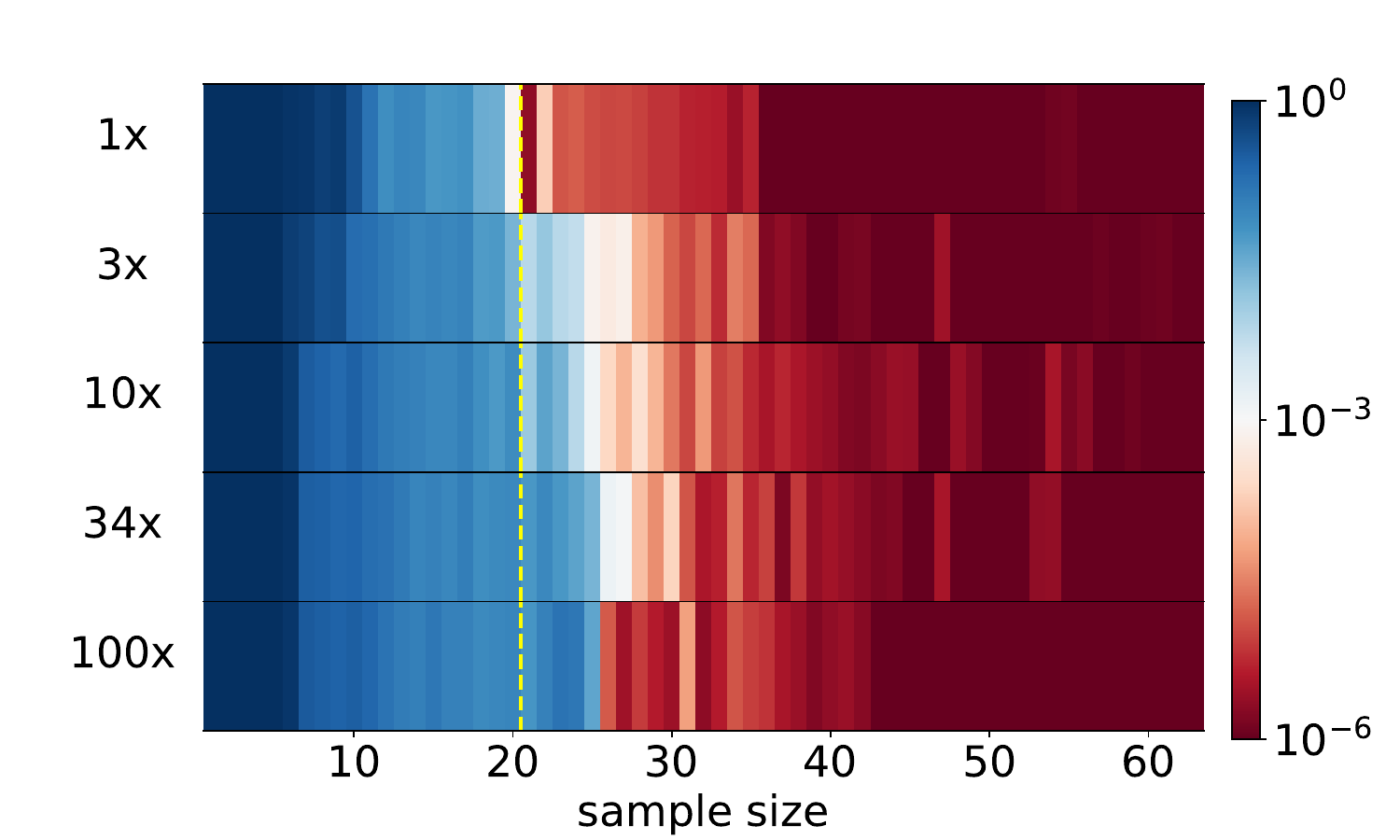}}\\
        \subfigure[CNN without weight sharing]
    {\includegraphics[width=0.49\textwidth]{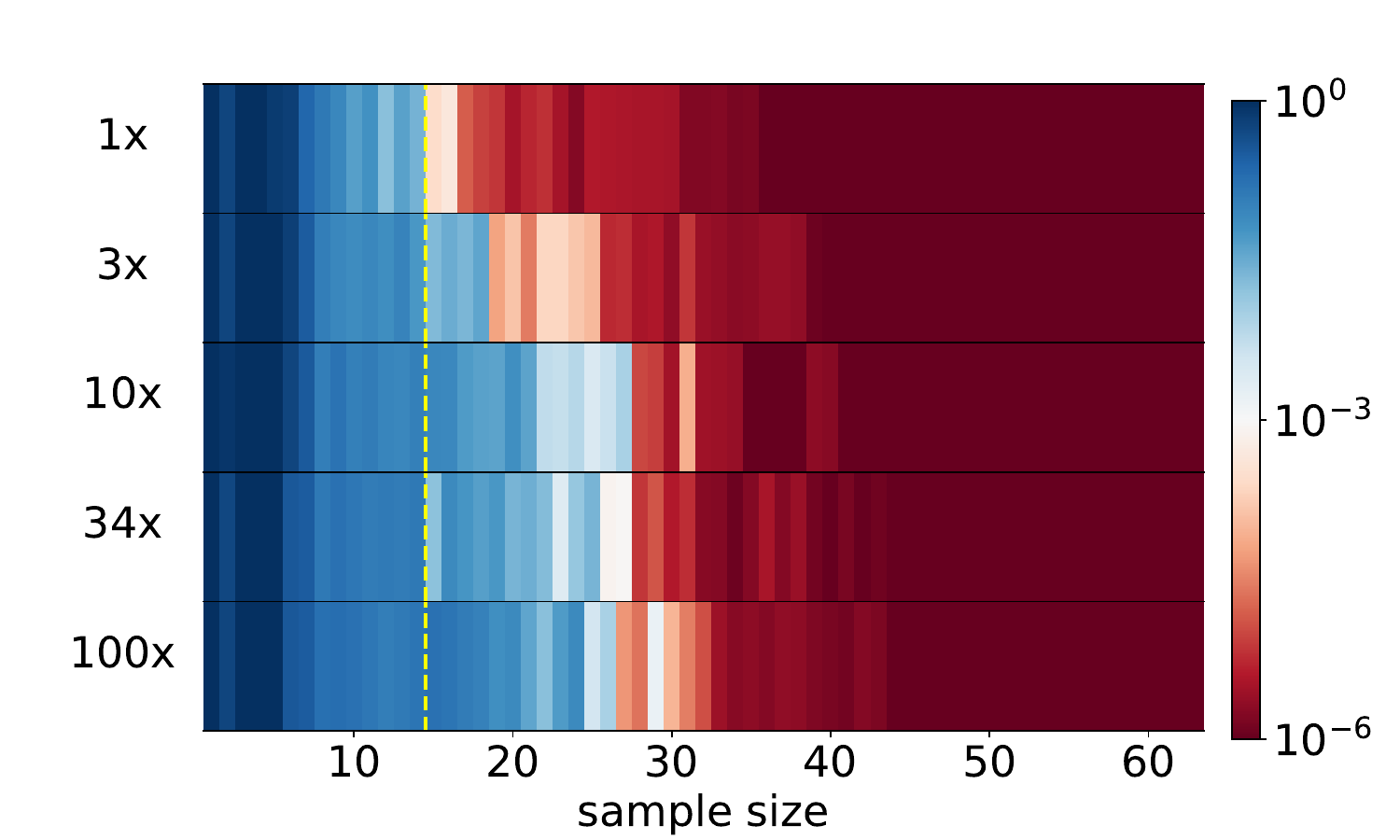}}
        \subfigure[CNN with weight sharing]
    {\includegraphics[width=0.49\textwidth]{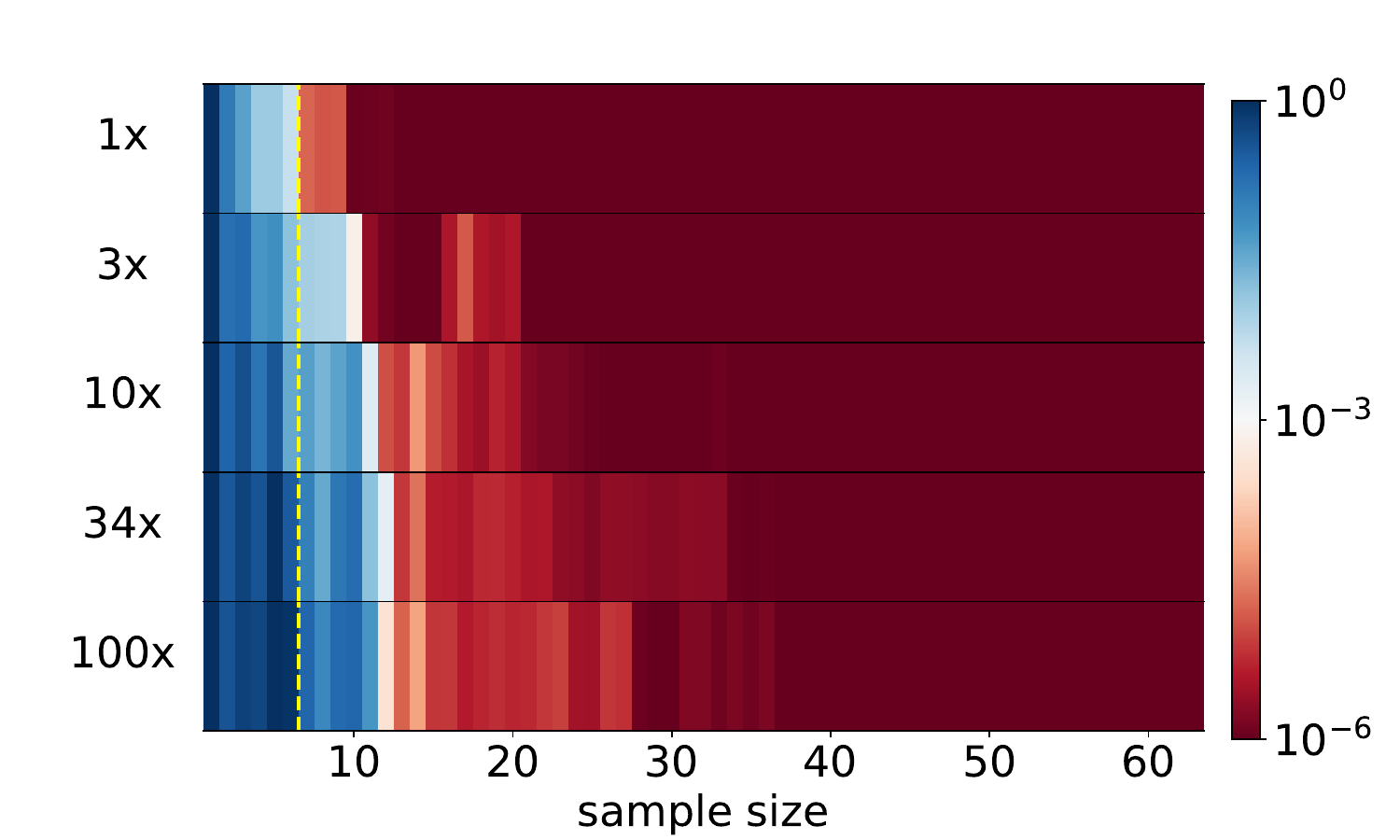}}

\caption{Average test error (color) for NNs of different architectures (ordinate)  and sample sizes (abscissa) in fitting the target function Eq. \eqref{eq:NN_target}. The yellow dashed line for each row indicates the model rank of the target in the corresponding NN. (a) Two-layer $1$-kernel tanh-CNN vs. two-layer $1$-kernel tanh-CNN without weight sharing vs. two-layer width-$3$ fully-connected tanh-NN. Note that these NNs are referred to as $1\textnormal{x}$ for each architecture in (b-d). (b) Two-layer $N$-kernel tanh-CNN, (c) two-layer $N$-kernel tanh-CNN without weight sharing, and (d) two-layer width-$3N$ fully-connected tanh-NN labeled by $N\textnormal{x}$ for $N=1,3,10,34,100$.
For all experiments, network parameters are initialized by a normal distribution with mean $0$ and variance $10^{-20}$, and trained by full-batch gradient descent with a fine-tuned learning rate. For the training dataset and the test dataset, we construct the input data through the standard normal distribution and obtain the output values from the target function. The size of the training dataset varies whereas the size of the test dataset is fixed to $1000$. The learning rate for the experiments in each setup is fine-tuned from $0.05$ to $0.5$ for a better generalization performance.
} 
\label{fig: network_stab}
\end{figure} 

\section{Conclusion}\label{sec:conclusion}

\begin{figure}
    \centering
    \includegraphics[width=0.9\textwidth]{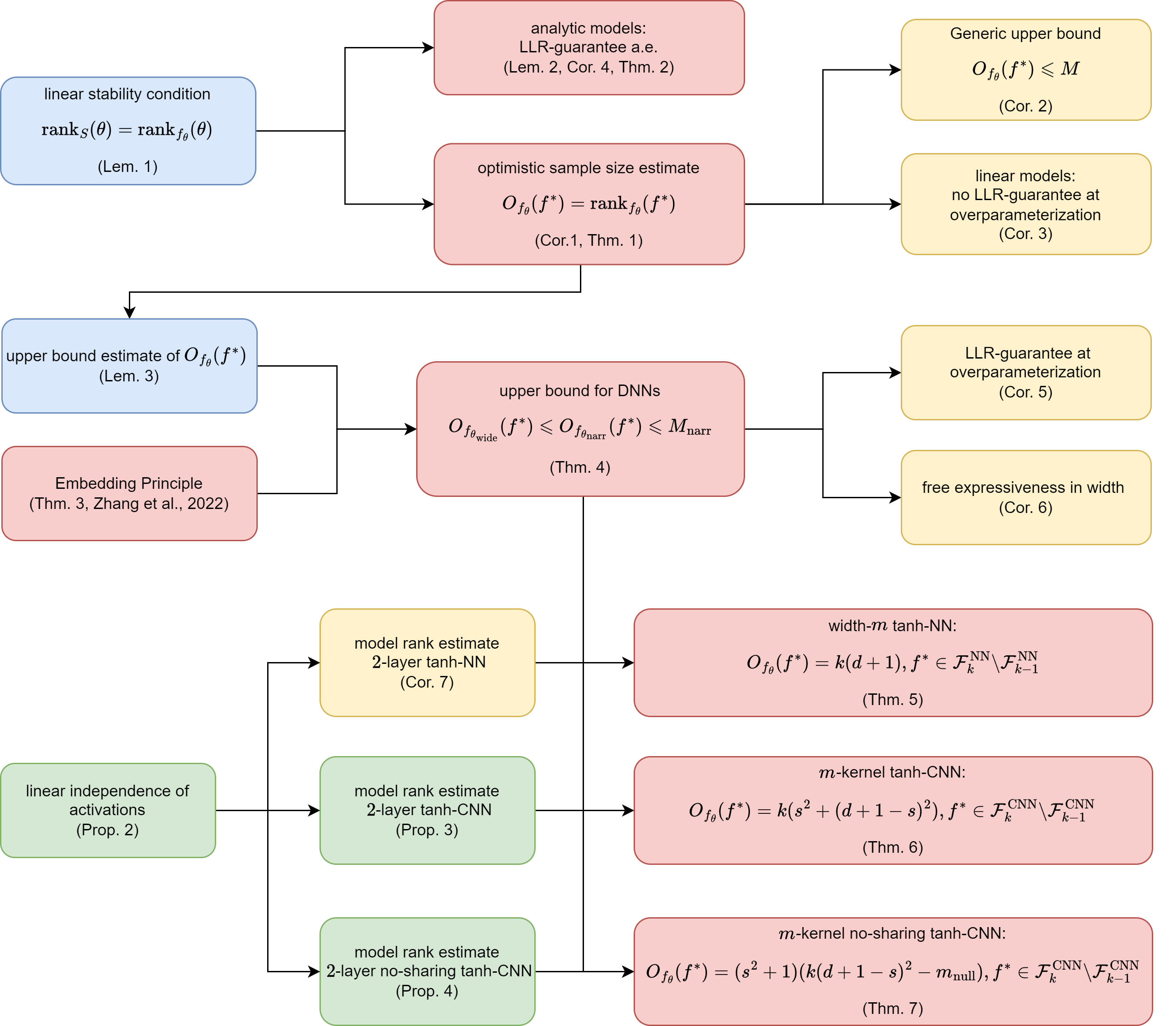}
    \caption{Schematic overview of our theoretical results and interconnections.}
    \label{fig:flow}
\end{figure}

In this study, we have established a local linear recovery (LLR) guarantee for deep neural networks (DNNs), demonstrating that all functions expressible by narrower DNNs possess an LLR-guarantee at overparameterization. Fig. \ref{fig:flow} presents a schematic overview of our theoretical results and their interconnections. Our work lays a solid groundwork for advancing the recovery theory of DNNs. Suggested by our results for two-layer NNs (Section \ref{sec: Optimistic sample sizes for two-layer NN}), the linear independence of neurons as in Proposition \ref{prop: linear independence of neurons} plays a key role in the exact calculation of model rank. For future research, three significant challenges remain open: (i) determining how to recover a target function with a sample size that approaches the optimistic estimate; (ii) investigating whether DNNs can offer stronger forms of recovery guarantees at overparameterization.  
(iii) exploring neuron independence in deeper networks to exactly estimate their model ranks.

More recently, progress has been made on these fronts. \cite{zhang2024implicit} presents empirical evidence suggesting that dropout techniques may enable the recovery of a target function with a sample size that aligns with its optimistic one. Furthermore, \cite{zhang2023structure} introduces a stronger form of recovery guarantee, namely the local recovery guarantee, for two-layer NNs at overparameterization. As we progress, we anticipate a significant deepening of our theoretical grasp of target recovery in DNNs.

\section*{Acknowledgement}
This work is sponsored by the National Key R\&D Program of China  Grant No. 2022YFA1008200, the National Natural Science Foundation of China Grant No. 12101402, the Lingang Laboratory Grant No.LG-QS-202202-08, Shanghai Municipal of Science and Technology Major Project No. 2021SHZDZX0102. We are grateful for the insightful comments provided by Zhi-Qin John Xu and Tao Luo.


\end{document}